\documentclass{article}
\usepackage{iclr2027_conference,times}
\iclrfinalcopy
\makeatletter\g@addto@macro\@maketitle{\lhead{Preprint}}\makeatother

\title{Stochastic Estimation of\\Transduced Language Models}

\usepackage[utf8]{inputenc}
\usepackage[T1]{fontenc}
\usepackage{microtype}
\usepackage{graphicx}
\usepackage{xcolor}
\usepackage{booktabs}
\usepackage{subcaption}
\usepackage{natbib}
\usepackage{url}
\usepackage{hyperref}
\usepackage{amsmath}
\usepackage{amssymb}
\usepackage{amsfonts}
\usepackage{mathtools}
\usepackage{amsthm}
\usepackage{ellipsis}

\usepackage[inline]{enumitem}
\setitemize{noitemsep,topsep=0pt,parsep=0pt,partopsep=0pt,leftmargin=*}
\setenumerate{noitemsep,topsep=0pt,parsep=0pt,partopsep=0pt,leftmargin=*}

\usepackage[capitalize,noabbrev]{cleveref} 
\usepackage{fancyvrb}
\crefname{FancyVerbLine}{line}{lines}
\makeatletter

\renewcommand{\p@FancyVerbLine}{}
\def\FV@refstepcounter#1{%
  \stepcounter{#1}%
  \protected@edef\@currentlabel{\arabic{FancyVerbLine}}%
  \cref@constructprefix{FancyVerbLine}{\cref@result}%
  \protected@edef\cref@currentlabel{[FancyVerbLine][\arabic{FancyVerbLine}][\cref@result]\arabic{FancyVerbLine}}%
}
\makeatother
\usepackage{cancel}
\usepackage{inconsolata}
\usepackage{mleftright}
\usepackage{xspace}
\usepackage[frozencache,cachedir=minted-cache]{minted}
\usepackage{tikz}
\usetikzlibrary{automata,arrows,arrows.meta,backgrounds,decorations.markings,decorations.pathmorphing,fit,matrix,positioning,shapes,shapes.arrows,shapes.geometric,shapes.misc}
\usepackage[flushmargin,hang]{footmisc}
\usepackage{bbm}
\usepackage{dsfont}
\usepackage{thmtools}
\usepackage{thm-restate}
\usepackage{float}
\usepackage{flafter}
\usepackage{relsize}
\usepackage{pifont}

\setminted{linenos,firstnumber=last,escapeinside=@@,numbersep=4pt}

\tikzset{
  CenterFix/.style={baseline={(current bounding box.center)}},
  TopFix/.style={baseline={(current bounding box.north)}},
}

\theoremstyle{plain}
\newtheorem{theorem}{Theorem}[section]
\declaretheorem[style=plain,sibling=theorem,name=Proposition]{proposition}

\definecolor{ETHGray}{RGB}{111,111,111}

\crefname{section}{\S}{\S\S}
\crefname{table}{Tab.}{Tab.}
\crefname{figure}{Fig.}{Figs.}
\crefname{algorithm}{Alg.}{}
\crefname{equation}{Eq.}{Eq.}
\crefname{appendix}{App.}{}
\crefname{theorem}{Theorem}{}
\crefname{prop}{Proposition}{}
\crefname{definition}{Def.}{}
\crefname{cor}{Corollary}{}
\crefname{observation}{Observation}{}
\crefformat{section}{\S#2#1#3}
\crefname{namedtheorem}{Hyp.}{Hypotheses}

\crefname{footnote}{footnote}{footnotes}
\Crefname{footnote}{Footnote}{Footnotes}

\definecolor{slatelake}{HTML}{6A8FA8}
\definecolor{saffron}{HTML}{C4A94D}
\definecolor{jademist}{HTML}{5A9E8F}
\hypersetup{
  colorlinks=true,
  linkcolor=blue!35!black,   
  citecolor=blue!35!black,   
  urlcolor=blue!35!black     
}
\tikzset{
transition/.style={-{Latex[length=1.8mm, width=1.4mm]}}, 
every state/.style={thick, fill=gray!10, text height=1.4ex, text depth=0.25ex},
every loop/.append style=-{Latex[length=1.8mm, width=1.4mm]},
initial text=$ $,
line width=0.7pt,
node distance=10mm,
minimum size=4mm,
}

\DeclareMathAlphabet{\mathdutchcal}{U}{dutchcal}{m}{n}

\newcommand{\mymacro}[1]{#1}

\renewcommand{\Pr}[2][]{\mathop{\mathrm{Pr}}_{#1}\mleft[#2\mright]}

\newcommand{\myMintedFunc}[1]{\mbox{\texttt{{\color{blue}#1}}}\xspace}
\newcommand{\isLive}{\myMintedFunc{is\_live}}
\newcommand{\isMember}{\myMintedFunc{is\_member}}
\newcommand{\isCylinder}{\myMintedFunc{is\_cylinder}}
\newcommand{\pruneFunc}{\myMintedFunc{prune}}
\newcommand{\pruneTau}{\prune_{\pruneThreshold}}
\newcommand{\algSMC}{\myMintedFunc{smc\_simple}}
\newcommand{\algSMCrb}{\myMintedFunc{smc\_rb}}
\newcommand{\algSWOR}{\myMintedFunc{beam\_swor}}
\newcommand{\algResample}{\myMintedFunc{resample}}
\newcommand{\algTopM}{\prune_M}
\newcommand{\sworB}{\myMintedFunc{beam\_swor\_adaptive}}
\newcommand{\sworHalt}{\myMintedFunc{swor\_adaptive}}
\newcommand{\algPPS}{\myMintedFunc{pps}}
\newcommand{\queuep}{\queue^{\prime}}
\newcommand{\algQuotient}{{\color{TokenColor}\texttt{Q}}}
\newcommand{\algRemainder}{{\color{TokenColor}\texttt{R}}}
\newcommand{\algBeam}{\myMintedFunc{beam\_summing}}
\newcommand{\swor}{\myMintedFunc{swor}}
\newcommand{\algQuotientp}{{\color{TokenColor}\texttt{Q}}'}
\newcommand{\algRemainderp}{{\color{TokenColor}\texttt{R}}'}
\newcommand{\algFixedWOR}{\myMintedFunc{systematic}}
\providecommand{\prune}{\myMintedFunc{prune}}

\makeatletter
\renewcommand*{\mathellipsis}{%
  \mathinner{%
    \kern\ellipsisbeforegap%
    {\ldotp}\kern\ellipsisgap%
    {\ldotp}\kern\ellipsisgap%
    {\ldotp}\kern\ellipsisaftergap%
  }%
}
\renewcommand*{\dotsb@}{%
  \mathinner{%
    \kern\ellipsisbeforegap%
    {\cdotp}\kern\ellipsisgap%
    {\cdotp}\kern\ellipsisgap%
    {\cdotp}\kern\ellipsisaftergap%
  }%
}
\renewcommand*{\@cdots}{%
  \mathinner{%
    \kern\ellipsisbeforegap%
    {\cdotp}\kern\ellipsisgap%
    {\cdotp}\kern\ellipsisgap%
    {\cdotp}\kern\ellipsisaftergap%
  }%
}
\renewcommand*{\ellipsis@default}{%
  \ellipsis@before
  \kern\ellipsisbeforegap
  .\kern\ellipsisgap
  .\kern\ellipsisgap
  .\kern\ellipsisgap
  \ellipsis@after\relax}
\renewcommand*{\ellipsis@centered}{%
  \ellipsis@before
  \kern\ellipsisbeforegap
  .\kern\ellipsisgap
  .\kern\ellipsisgap
  .\kern\ellipsisaftergap
  \ellipsis@after\relax}
\AtBeginDocument{%
  \DeclareRobustCommand*{\dots}{%
    \ifmmode\@xp\mdots@\else\@xp\textellipsis\fi}}
\def\ellipsisgap{-.05em}
\def\ellipsisbeforegap{-.05em}
\def\ellipsisaftergap{.05em}
\makeatother

\newcommand{\justification}[1]{%
    \refstepcounter{equation}%
    \tag{\theequation \textcolor{black!50}{, \footnotesize{#1}}}
}

\newcommand{\defn}[1]{\textbf{#1}}
\newcommand{\defeq}{\mathrel{\overset{\raisebox{-0.25ex}{\textnormal{\tiny def}}}{=}}}
\newcommand{\simIID}{\mathrel{\overset{\raisebox{-0.25ex}{\textnormal{\tiny i.i.d.}}}{\sim}}}
\newcommand{\indicator}[1]{\mathbbm{1}\mleft\{#1\mright\}}
\newcommand{\bigo}{{\mymacro{ \mathcal{O}}}}
\newcommand{\bigO}[1]{{\mymacro{ \bigo\mleft(#1\mright)}}}

\newcommand{\Quotient}{{\color{TokenColor}\mathcal{Q}}}
\newcommand{\Remainder}{{\color{TokenColor}\mathcal{R}}}
\newcommand{\preCover}{{\color{TokenColor}\mathcal{P}}}
\newcommand{\maxQuotient}{{\color{TokenColor}\mathcal{C}}}
\newcommand{\eos}{{\mymacro{\textsc{eos}}}}
\newcommand{\basisSet}[1]{\langle #1\rangle}
\definecolor{CharacterColor}{RGB}{98,115,19}
\definecolor{TokenColor}{RGB}{140,10,89}
\newcommand{\charfont}[1]{{\color{CharacterColor}\ifmmode\else\strut\fi\texttt{#1}}}
\newcommand{\rvX}[0]{{\color{TokenColor}X}}

\newcommand{\transFnSym}[0]{f}
\newcommand{\transFn}{\ensuremath{{\color{CharacterColor}\transFnSym}}\xspace}
\newcommand{\STPTB}{{\transST}_{\mathrm{ptb}}}
\newcommand{\transFnDNA}{\ensuremath{\transFn_{\mathrm{dna2aa}}}}
\newcommand{\STDNA}{{\transST}_{\mathrm{dna2aa}}}
\newcommand{\srcSym}{{\color{TokenColor}x}}
\newcommand{\srcSymp}{{\color{TokenColor}x^\prime}}
\newcommand{\srcSympp}{{\color{TokenColor}x^{\prime\prime}}}
\newcommand{\srcStr}
{\ensuremath{{\color{TokenColor}\boldsymbol{x}}}\xspace}
\newcommand{\srcStrp}[0]{{\color{TokenColor}\srcStr^\prime}}
\newcommand{\srcStrpp}[0]{{\color{TokenColor}\srcStr^{\prime\prime}}}
\newcommand{\srcStrLt}{\srcStr_{<t}}
\newcommand{\gpt}{\mymacro{p_{\mathsf{gpt2}}}}
\newcommand{\pbigram}{\mymacro{p_{\mathsf{bigram}}}}
\newcommand{\pdna}{\mymacro{p_{\mathsf{dna}}}}
\newcommand{\tgtSym}{{\color{CharacterColor}y}}
\newcommand{\tgtStr}{{\color{CharacterColor}\boldsymbol{y}}}
\newcommand{\pSource}{\mymacro{p_{\srcAlphabet}}}
\newcommand{\pTarget}{\mymacro{p_{\tgtAlphabet}}}
\DeclareRobustCommand{\prefixLanguageOp}[1]{\overrightarrow{#1}}
\newcommand{\prefixSource}{{\mymacro{\prefixLanguageOp{\pSource}}}}
\newcommand{\prefixTarget}{{\mymacro{\prefixLanguageOp{\pTarget}}}}
\newcommand{\srcAlphabet}[0]{{\color{TokenColor}\ensuremath{\mathcal{X}}}\xspace}
\newcommand{\tgtAlphabet}[0]{{\color{CharacterColor}\ensuremath{\mathcal{Y}}}\xspace}
\newcommand{\srcStrings}[0]{\ensuremath{\srcAlphabet^*}\xspace}
\newcommand{\tgtStrings}[0]{\ensuremath{\tgtAlphabet^*}\xspace}
\newcommand{\transST}{{\color{CharacterColor}\texttt{f}}}
\newcommand{\States}{\mymacro{\mathsf{S}}}
\newcommand{\state}{\mymacro{s}}

\newcommand{\srcStringsSub}{{\color{TokenColor}Z}}
\newcommand{\pr}{{\color{TokenColor}\operatorname{pf}}}
\newcommand{\queue}[0]{\mathdutchcal{P}}
\newcommand{\queueBar}{\overline{\queue}}
\newcommand{\VBar}{\overline{V}}
\newcommand{\threshold}{\eta}
\newcommand{\pruneThreshold}{\tau}
\newcommand{\idfunc}{\mathds{1}}

\newcommand{\w}{\mymacro{w}}
\newcommand{\weight}{\w}
\newcommand{\potential}{\mymacro{\phi}}
\newcommand{\rvSym}{{\color{TokenColor}X}}
\newcommand{\proposal}{\mymacro{q}}
\newcommand{\prefixProposal}{\mymacro{\overrightarrow{\proposal}}}
\newcommand{\prefixProposalt}{\mymacro{\overrightarrow{\proposal_t}}}
\newcommand{\expected}[2]{\mymacro{\mathbb{E}_{#1}[#2]}}
\providecommand{\liveset}[1]{\mathcal{A}(#1)}
\providecommand{\cert}{\mathcal{K}}
\providecommand{\interTarget}{\mu}
\providecommand{\twist}{\psi}
\providecommand{\twistOptimal}{\twist^{*}}
\newcommand{\evZ}{Z}
\newcommand{\propZ}{\nu}
\newcommand{\liveZ}{\lambda}
\newcommand{\estZ}{\hat{Z}}
\newcommand{\accum}{\mymacro{A}}
\newcommand{\swsum}{\widehat{\Sigma}}
\newcommand{\cmark}{\textcolor{green!60!black}{\ding{51}}}
\newcommand{\xmark}{\textcolor{red!70!black}{\ding{55}}}
\providecommand{\significance}{p}
\providecommand{\dll}{\Delta\mathrm{LL}}

\renewcommand{\eth}{\textmd{\algQuotient}}
\newcommand{\ku}{\textmd{\algRemainder}}
\newcommand{\chifro}{\chi}

\author{%
\textbf{Vésteinn Snæbjarnarson}$^{\eth\ku}$~~~~
\textbf{Samuel Kiegeland}$^{\eth\chifro}$~~~~
\textbf{Manuel de Prada Corral}$^{\eth}$ \\
\textbf{Ryan Cotterell}$^{\eth}$~~~~
\textbf{Tim Vieira}$^{\eth}$
\\
$^\eth$ETH Z\"{u}rich\quad
$^\ku$University of Copenhagen \quad $^\chifro$CHI-FRO\\
\texttt{\{\href{mailto:vest.snae@gmail.com}{vest.snae},
\href{mailto:samuel.kiegeland@gmail.com}{samuel.kiegeland},
\href{mailto:manueldeprada@gmail.com}{manueldeprada},
\href{mailto:tim.f.vieira@gmail.com}{tim.f.vieira}\}@gmail.com}\\
\texttt{\href{mailto:ryan.cotterell@inf.ethz.ch}{ryan.cotterell}@inf.ethz.ch} \\
\vspace{-0.3cm}
}

\usepackage{etoc}
\let\realaddcontentsline\addcontentsline 
\begin{document}

\maketitle

\begin{abstract}
  Transduced language models (TLMs) compose a pretrained \emph{source} language model with a functional finite-state transducer to induce a language model over \emph{target} strings.
  Computing the probability of a target prefix under a TLM amounts to summing the source-model probabilities of all source strings that the transducer maps to target strings beginning with that prefix.
  This set can be exponentially large or infinite.
  Prior work uses a computational shortcut based on source prefix probabilities, then approximates the resulting sum with threshold-pruned beam summing.
  This produces a lower bound with unknown error.
  Instead, we resample source prefixes without replacement and reweight each selected prefix by the inverse of its inclusion probability.
  We show that applying this correction recursively gives an unbiased estimator of the target prefix probability and lets us estimate the mass lost by threshold pruning.
  Our beam-summing algorithm extends the retained source prefixes and samples which prefixes to keep, reducing their number as more probability mass is added to the running estimate.
  This can save computation and guarantees that the run halts with probability one.
  We evaluate the method on encyclopedic text and DNA against sequential Monte Carlo baselines that resample with replacement.
  It achieves a better compute--variance tradeoff on text and lower error at the same maximum number of particles on DNA.
  On a DNA-to-amino-acid transduction, it reduces runtime by several orders of magnitude relative to threshold-pruned beam summing and makes estimating prefix probabilities for long target strings feasible.
  Replacing threshold pruning with unbiased sampling in a published reading-time analysis substantially lowers the estimated corpus surprisal but leaves the published conclusions unchanged.
\end{abstract}

\etocdepthtag.toc{mainmatter}

\section{Introduction}
\label{sec:intro}
Language models define distributions over strings, while applications may require distributions induced by deterministic transformations of those strings.
\citet{2025transducing} call this the \emph{string mismatch problem}.
They address it by composing a pretrained source model with a functional finite-state transducer, inducing a language model over target strings without retraining, a \emph{transduced language model} (TLM).
This makes it possible, for example, to compute word-level surprisal for reading-time prediction \citep{kiegeland2025proper} or to obtain a model over amino-acid strings from one defined over DNA strings.
The TLM's next-symbol probabilities are obtained from its target prefix probabilities. Each target prefix probability amounts to summing the source-model probabilities of all source strings that the transducer maps to target strings with that prefix.
This set of source strings is the target prefix's \emph{precover}.\looseness=-1

\citet{2025transducing} derive a computational shortcut for enumerating and marginalizing over the precover.
They implement it as a beam-summing algorithm that enumerates source prefixes breadth-first.
When every extension of a prefix belongs to the precover, the algorithm adds the prefix's entire probability mass to the running estimate.
Otherwise, if the prefix itself belongs to the precover as a complete source string, it adds that string's probability.
To keep the calculation manageable, \citet{2025transducing} prune low-weight prefixes along the way.
With a sufficiently low threshold, beam summing can closely approximate the target prefix probability when most of the relevant mass is carried by a relatively small set of prefixes \citep{2025transducing,vieira2024languagemodelstokenslanguage}.\looseness=-1

To evaluate the impact of threshold pruning, the prior work compares normalized next-symbol distributions across thresholds, revealing changes in relative probabilities as the threshold varies.
This does not, however, quantify the total mass discarded by pruning.
The computation also becomes difficult when mass is spread across many similarly weighted prefixes.
Retaining most of the mass then requires tracking many prefixes, and a practical threshold can discard a non-negligible amount.
In such settings, the pruning error is unknown, and threshold-pruned beam summing can be too expensive for useful target lengths.\looseness=-1

This work replaces threshold pruning with resampling without replacement.
We sample distinct prefixes using weight-based inclusion probabilities and reweight each selected prefix by the inverse of its inclusion probability \citep{horvitz1952generalization,fearnhead2003online,Shah2018}.
This preserves each prefix's weight in expectation without changing the rest of the beam summing algorithm.
We also reduce the number of retained prefixes as more probability mass is added to the running estimate, saving computation.\looseness=-1

We show that this reweighting gives an unbiased estimator of the target prefix probability as long as the run halts with probability one (\cref{thm:swor-unbiased}).
We then show that adaptively reducing the number of retained prefixes preserves unbiasedness and that the resulting run halts with probability one (\cref{prop:sworb}).
Repeated runs estimate the sampling variation of the unbiased estimator, whereas deterministic threshold-pruned beam summing returns the same value every time.
At each target prefix, the difference between the mean of the unbiased estimates and the threshold-pruned value is an estimate of the mass lost through pruning.

The experiments show that the method makes estimating prefix probabilities along longer target strings feasible under a high-entropy source model and quantifies how pruning bias affects a downstream analysis.
On the DNA-to-amino-acid transduction, threshold-pruned beam summing exceeds a five-minute limit at target position $13$, whereas the sampling-based version reaches position $200$ in about five seconds (\cref{fig:dna-scaling}).
In psycholinguistics, \citet{kiegeland2025proper} use a TLM to derive word-level surprisal from a token-level language model for reading-time prediction.
Replacing threshold pruning with unbiased sampling changes the estimated corpus surprisal by $106$ nats but leaves the published conclusions unchanged (\cref{sec:psycholing}).
We compare against two unbiased sequential Monte Carlo (SMC) baselines that extend each retained prefix by one sampled live source symbol and use resampling with replacement.
On encyclopedic text, resampling without replacement achieves a better compute--variance tradeoff than the SMC baselines.
On DNA, it achieves lower error at the same maximum number of particles.
Comparisons with exact target prefix probabilities validate the estimator where those values are available (\cref{sec:experiments}).

\section{Background}
We review the background and notation for language models and transduced language models, largely following the definitions of \citet{2025transducing}.\footnote{\Cref{sec:glossary} collects all notation in one place.}\label{sec:background}

\subsection{Language Models}
\label{sec:lms}
Let $\srcAlphabet$ be a set of symbols called an \defn{alphabet}, and $\srcStrings$ be the set of all finite strings of those symbols. We write $\srcStrp\preceq\srcStr$, equivalently $\srcStr\succeq\srcStrp$, when $\srcStrp$ is a prefix of $\srcStr$. A \defn{cylinder} with basis $\srcStr$ is $\basisSet{\srcStr}\defeq\{ \srcStr\srcStrp \mid \srcStrp\in \srcStrings\}$. We define the \defn{prefix-free} operation $\pr(\srcStringsSub)$ as $\pr(\srcStringsSub) \defeq \{ \srcStr \in \srcStringsSub \colon \nexists\, \srcStr' \in \srcStringsSub \text{ such that } \srcStr' \prec \srcStr \}$.

A \defn{language model} $\pSource$ is a distribution over $\srcStrings$. With $\eos\notin\srcAlphabet$ a special \defn{end-of-sequence symbol}, we define the \defn{prefix probability} and the \defn{conditional prefix probabilities} as
\begin{equation}
  \label{eq:prefix-prob}
  \prefixSource(\srcStr)\defeq \sum_{\srcStrp\in\srcStrings}\pSource(\srcStr\srcStrp),
  \qquad
  \prefixSource(\srcStrp\mid \srcStr) \defeq \frac{\prefixSource(\srcStr\srcStrp)}{\prefixSource(\srcStr)},
  \qquad
  \prefixSource(\eos\mid \srcStr)\defeq \frac{\pSource(\srcStr)}{\prefixSource(\srcStr)}.
\end{equation}
When $\prefixSource(\srcStr)=0$, we take $\prefixSource(\srcStrp\mid \srcStr)=0$ and $\prefixSource(\eos\mid \srcStr)\defeq 1$.
An interface to the language model is then given by $\pSource(\srcStr) = \prefixSource(\eos\mid \srcStr) \prod_{t=1}^{|\srcStr|}\prefixSource(\srcStr_{t}\mid \srcStrLt)$.
\footnote{Because a language model places all its mass on finite strings, ancestral sampling from these conditionals reaches $\eos$ with probability one, and the mass beyond depth $n$, $\sum_{|\srcStr|=n}\prefixSource(\srcStr)$, vanishes as $n$ grows.
  In practice an autoregressive model supplies the conditionals directly, and they define a language model exactly when their ancestral sampler terminates with probability one. Such conditionals are called \defn{tight} \citep{cotterell2024formalaspectslanguagemodeling}.\looseness=-1}
When $\prefixSource(\srcStr)>0$, the conditionals also give the distribution over the \defn{completions} of a prefix, $\pSource(\srcStrpp\mid\srcStr)\defeq\pSource(\srcStr\srcStrpp)/\prefixSource(\srcStr)$ for $\srcStrpp\in\srcStrings$, written $\srcStrpp\sim\pSource(\cdot\mid\srcStr)$.

\subsection{Transduced Language Models}
\label{sec:tlms}

A \defn{transduced language model} is a language model $\pTarget$ characterized by the tuple $(\pSource, \transST)$ where $\transST$ is a transducer encoding a function $\transFn\colon \srcStrings\to \tgtStrings$. Let $\rvX$ be a random variable drawn from $\pSource$. For $\tgtStr\in\tgtStrings$, we then define
\begin{equation}
  \pTarget(\tgtStr) \defeq \Pr[\rvX\sim\pSource]{\tgtStr=\transFn(\rvX)} = \sum_{\srcStr\in\srcStrings}\idfunc\{\transFn(\srcStr)=\tgtStr\}\pSource(\srcStr),
\end{equation}
and its prefix probability
\begin{equation}
  \label{eq:tlm-prefix-prob}
  \prefixTarget(\tgtStr) \defeq \Pr[\rvX\sim\pSource]{\tgtStr\preceq\transFn(\rvX)} =  \sum_{\srcStr\in\srcStrings}\idfunc\{\tgtStr\preceq\transFn(\srcStr)\}\pSource(\srcStr).
\end{equation}
\citet{2025transducing} give a computational shortcut by defining the \defn{quotient} and \defn{remainder}. For $\tgtStr\in \tgtStrings$, we define the \defn{precover} of $\tgtStr$ with respect to $\transFn$ as $\preCover(\tgtStr) \defeq \{ \srcStr \in \srcStrings\colon \tgtStr \preceq \transFn(\srcStr) \}$. Let $
  \maxQuotient(\tgtStr) \defeq \left\{ \srcStr \in \srcStrings \colon \basisSet{\srcStr} \subseteq \preCover(\tgtStr) \right\}$, then $\maxQuotient(\tgtStr)$ is the largest union of cylinders $\basisSet{\srcStr}$ (\cref{sec:lms}) contained in $\preCover(\tgtStr)$. The quotient and remainder are  $\Quotient(\tgtStr) \defeq \pr(\maxQuotient(\tgtStr))$ and $\Remainder(\tgtStr) \defeq \preCover(\tgtStr) \setminus \maxQuotient(\tgtStr)$.
The computational shortcut is the \defn{quotient--remainder decomposition}:
\begin{equation}
  \label{eq:simple-qr}
  \prefixTarget(\tgtStr) = \sum_{\srcStr \in \Remainder(\tgtStr)} \pSource(\srcStr) \;+\!\! \sum_{\srcStr \in \Quotient(\tgtStr)} \prefixSource(\srcStr).
\end{equation}
\citet{2025transducing} show that the decomposition counts every source string exactly once, since the precover splits as the disjoint union $\preCover(\tgtStr)=\bigl(\bigcup_{\srcStr\in\Quotient(\tgtStr)}\basisSet{\srcStr}\bigr)\sqcup\Remainder(\tgtStr)$.
Their algorithms use the predicates below to determine whether a given prefix belongs to the quotient or the remainder.\looseness=-1

\paragraph{Precover predicates.}
The three predicates classify a source prefix $\srcStr$ relative to $\preCover(\tgtStr)$:
\begin{description}
  \item \defn{Live}, $\isLive(\srcStr,\tgtStr)$: some extension covers $\tgtStr$, that is $\exists\,\srcStrp\in\srcStrings\colon\srcStr\srcStrp\in\preCover(\tgtStr)$.
  \item \defn{Member}, $\isMember(\srcStr,\tgtStr)$: $\srcStr\in\preCover(\tgtStr)$, the output emitted so far covers $\tgtStr$.
  \item \defn{Cylinder}, $\isCylinder(\srcStr,\tgtStr)$: $\srcStr\srcStrp\in\preCover(\tgtStr)$ for every $\srcStrp\in\srcStrings$, every extension covers $\tgtStr$.
\end{description}
These predicates locate the decomposition members \cref{eq:simple-qr}: the first cylinder prefix on a path generates a cylinder of $\maxQuotient(\tgtStr)$ and so belongs to the quotient $\Quotient(\tgtStr)$, a member that is not a cylinder is a remainder element of $\Remainder(\tgtStr)$, and a non-live prefix is dropped.
Finite-state transducers make these semantic tests computable.
When $\transFn$ is encoded by a finite-state transducer, the Live, Member, and Cylinder predicates can be evaluated in finite time.
Following \citet{2025transducing}, the implementation evaluates them through lazy determinization, expanding only the states reached as it enumerates source prefixes.
The algorithms otherwise require only these three predicates and apply to any representation that can compute them.

\section{Unbiased Estimation of Transduced Language Models}
\label{sec:swor}

\begin{figure}[t]
  \centering
  \smaller
  \begin{minipage}[t]{0.48\linewidth}
    \begin{minted}[autogobble,xleftmargin=10pt]{python}
def @$\algBeam(\tgtStr, \twist, \pruneFunc)$@:
  @$N \gets |\tgtStr|$@
  if @$N = 0$@:
    @$(\algQuotientp, \algRemainderp), \estZ \gets (\{(\varepsilon, 1)\},\ \emptyset),\ [\,]$@
  else:
    @$(\algQuotientp, \algRemainderp), \estZ  \gets \algBeam(\tgtStr_{<N}, \twist, \pruneFunc)$@
  @$\queue \gets \textsc{Queue}(\algQuotientp)$@
  @$(\algQuotient, \algRemainder), \accum \gets (\emptyset, \emptyset),\ 0$@
  for @$(\srcStr, \weight) \in \algRemainderp$@:
    if @$\isMember(\srcStr, \tgtStr)$@: @\label{line:targetstep}@
      @$\algRemainder$@.add(@$(\srcStr, \weight)$@)
      @$\accum \texttt{ += } \weight$@
  while @$|\queue| > 0$@:
    @$\queuep \gets \emptyset$@
    for @$\srcStr,\weight \in \queue$@: @\label{line:process}@
      if @$\isCylinder(\srcStr, \tgtStr)$@: @\label{line:cylindercheck}@
        @$\algQuotient$@.add(@$(\srcStr, \weight)$@)
        @$\accum \texttt{ += } \weight$@
        continue @\label{line:continueabs}@
      if @$\isMember(\srcStr, \tgtStr)$@: @\label{line:membercheck}@
        @$\algRemainder$@.add(@$(\srcStr, \weight\cdot \prefixSource(\eos \mid \srcStr))$@)
        @$\accum \texttt{ += } \weight\cdot \prefixSource(\eos \mid \srcStr)$@
      for @$\srcSymp \in \srcAlphabet$@ with @$\twist(\srcStr\srcSymp) > 0$@: @\label{line:enumAlphabet}@
        @$\queuep$@.add(@$(\srcStr \srcSymp,\ \weight \cdot \prefixSource(\srcSymp \mid \srcStr))$@) @\label{line:livecheck}@
    @$\queue \gets \pruneFunc(\queuep, \accum)$@
  @$\estZ$@.append(@$\accum$@)
  return @$(\algQuotient, \algRemainder), \estZ$@
\end{minted}
  \end{minipage}
  \hfill
  \begin{minipage}[t]{0.48\linewidth}
    \begin{minted}[autogobble,xleftmargin=10pt]{python}
def @$\algTopM(\queuep)$@:
  # Standard beam search
  @$((\srcStr^{(i)},\weight^{(i)}))_{i=1}^{N} \gets$@ sort_descending(@$\queuep$@, key=@$\weight$@)
  return @$\{(\srcStr^{(i)},\weight^{(i)})$@ for @$i = 1, \ldots, \min(M, N)\}$@
\end{minted}

    \begin{minted}[autogobble,xleftmargin=10pt]{python}
def @$\pruneTau(\queuep)$@:
  # Keep heaviest particles holding
  # @$\ge 1-\pruneThreshold$@ of the mass
  if @$|\queuep| = 0$@: return @$\emptyset$@
  @$((\srcStr^{(i)},\weight^{(i)}))_{i=1}^{N} \gets$@ sort_descending(@$\queuep$@, key=@$\weight$@)
  @$W_i \gets \textstyle\sum_{j \le i} \weight^{(j)}$@ for @$i = 1, \ldots, N$@
  @$k \gets$@ smallest @$i$@ with @$W_i \ge (1 - \pruneThreshold)\, W_N$@
  return @$\{(\srcStr^{(i)},\weight^{(i)})$@ for @$i = 1, \ldots, k\}$@
\end{minted}
  \end{minipage}
  \caption{%
    \textbf{Beam summing} for transducing language models (\cref{sec:two-algorithms}), the prior work's enumeration \citep{2025transducing} made to return the estimate, with the deterministic prunes it runs. \algBeam{} hands each \pruneFunc{} call the pool and the current-position running total $\accum$. A prune takes only the arguments it needs.
    The twist gates the extension step (\cref{line:enumAlphabet}) and can weight a prune's draw. $\twist=\isLive$ keeps exactly the particles such that $\isLive(\srcStr,\tgtStr)$ holds. \algBeam{} with $\algTopM$, keeping the $M$ heaviest particles, is standard beam search, and $\pruneTau$ keeps the heaviest particles holding at least $1-\tau$ of the pool's mass, for some fixed $M$ and $\tau$.
  }
  \label{fig:algos}
\end{figure}

The quotient--remainder decomposition (\cref{eq:simple-qr}) writes the prefix probability $\prefixTarget(\tgtStr)$ of a transduced language model (\cref{sec:background}) as a sum of source-model mass over the precover $\preCover(\tgtStr)$.
That sum is often far too large to enumerate, so the beam summing of \citet{2025transducing} applies $\pruneTau$, discarding mass to keep the calculation manageable.
We therefore replace deterministic pruning with resampling without replacement among the source prefixes produced at each extension step, using \emph{inclusion-probability} reweighting to preserve their weights in expectation (\cref{sec:swor-estimator}).
We write $\evZ\defeq\prefixTarget(\tgtStr)$ for this estimand and seek an estimate $\estZ$ that meets three criteria:
\begin{enumerate}[leftmargin=1.0cm]
  \item[(i)] it is unbiased, i.e., $\expected{}{\estZ}=\evZ$
  \item[(ii)] it has low variance
  \item [(iii)] it comes at a reasonable and configurable computational cost
\end{enumerate}
As a consequence of (iii), we require that the algorithms halt \emph{almost surely} (with probability 1).
This lets us estimate prefix probabilities whose quotient--remainder decomposition (\cref{sec:tlms}) is infinite, something the prior work cannot guarantee.\looseness=-1

We build this estimator starting with the prior work's algorithm, which computes the decomposition in \cref{sec:tlms}, where $\pruneTau{}$ is its only lossy step (\cref{sec:two-algorithms}).
We then replace $\pruneTau{}$ with \defn{sampling without replacement} (\defn{SWOR}) (\cref{sec:swor-estimator}), which samples which distinct source prefixes survive rather than discarding them deterministically.
The algorithm adds a prefix's quotient contribution when the \isCylinder{} predicate holds and, otherwise, its remainder contribution when the \isMember{} predicate holds.
Adding the remainder contribution directly rather than sampling the stop decision is a variance-reducing Rao--Blackwellization.
We extend the SWOR-based pruning function to use the running total to reduce the number of retained source prefixes as mass is accumulated, saving compute and ensuring almost-sure halting (\cref{sec:sworb}).\looseness=-1

\subsection{Beam Summing over the Precover Decomposition}
\label{sec:two-algorithms}
The algorithm in \cref{fig:algos} is the beam summing approach of the prior work made to return the estimand directly, here named \algBeam{} \citep{2025transducing}.
The recursion computes the quotient--remainder decomposition (\cref{eq:simple-qr}), one call per target prefix.
It adds cylinder elements to $\algQuotient$ (\defn{quotient contributions}) and member stop branches to $\algRemainder$ (\defn{remainder contributions}).
We call the queue's entries \defn{particles}, each a source prefix $\srcStr$ carrying a nonnegative \defn{weight} $\weight$, and the queue itself the \defn{pool}.
The queues and the collections $\algQuotient$ and $\algRemainder$ are multisets.
With no pruning, a particle's weight is exactly its source prefix mass $\prefixSource(\srcStr)$ (\cref{line:livecheck}).
Within each call, \algBeam{} enumerates source prefixes breadth-first.
The practical implementation of \citet{2025transducing} uses $\pruneTau$ to keep the enumeration manageable by deterministically keeping at least $1-\pruneThreshold$ of the mass, for some threshold $0<\tau<1$, potentially biasing the estimate by discarding positive covering mass (\cref{tab:tlm-zoo}, \cref{fig:algos}).\looseness=-1

Each nonempty call to \algBeam{} first recurses on the target prefix one position shorter, then initializes its pool from the returned quotient entries and passes the returned remainder entries through a membership re-check (\cref{line:targetstep}).
Within a call, each iteration processes every particle in the pool, forms child particles by extending non-cylinder prefixes by one source symbol, and updates the running total $\accum$ using the predicate checks (\cref{line:cylindercheck,line:membercheck}).
It is passed to \pruneFunc{} after each iteration, and when the pool empties, the call appends $\accum$ to the returned list $\estZ$ (\cref{eq:pos-estimate}).
\citet{2025transducing} enqueue only source prefixes that pass the Live predicate.
\algBeam{} generalizes this restriction through a \defn{twist} $\twist$, a nonnegative weighting on source prefixes.\footnote{The name follows the twisted particle filters of \citet{whiteley2014twisted}, a look-ahead factor multiplying, \emph{twisting}, the weights a filter tracks (\citealp{heng2020controlled,zhao2024smc}).}
At every position $t$, a twist must be positive wherever the optimal twist $\twistOptimal_t$ defined below is positive, so it zeroes only children with no covering mass ahead.
The extension step only creates child particles with $\twist(\srcStr\srcSymp)>0$ (\cref{line:enumAlphabet}).
The stochastic prunes introduced below use the twisted weights $\weight\,\twist(\srcStr)$ to choose survivors (\cref{sec:swor-estimator}).
Setting $\twist=\isLive$ and using $\pruneTau$ as the pruning function recovers the algorithm of \citet{2025transducing}.

At position $t$, let $\algQuotient_t$ and $\algRemainder_t$ be the quotient and remainder entries collected by the call for $\tgtStr_{\le t}$.
Grouping these entries by source prefix and dividing the remainder entries by their stop factor defines
\begin{equation}
  \label{eq:V-def}
  V_t(\srcStr)\defeq\smashoperator{\sum_{(\srcStrp,\,\weight)\,\in\,\algQuotient_t}}\weight\,\idfunc\{\srcStrp=\srcStr\}\;\;+\sum_{(\srcStrp,\,\weight)\,\in\,\algRemainder_t}\frac{\weight}{\prefixSource(\eos\mid\srcStr)}\,\idfunc\{\srcStrp=\srcStr\},
\end{equation}
the per-prefix estimate of $\prefixSource(\srcStr)$ produced by the call for $\tgtStr_{\le t}$.
A remainder summand is defined to be zero when $\prefixSource(\eos\mid\srcStr)=0$, in which case its contribution to $\accum$ is also zero.
Since the quotient and remainder entry weights are the only terms added to $\accum$, grouping them by source prefix gives
\begin{equation}
  \label{eq:pos-estimate}
  \estZ_t \defeq \smashoperator{\sum_{\srcStr\in\Remainder(\tgtStr_{\le t})}}V_t(\srcStr)\,\prefixSource(\eos\mid\srcStr)\;\;+\smashoperator{\sum_{\srcStr\in\Quotient(\tgtStr_{\le t})}}V_t(\srcStr),
\end{equation}
the quotient--remainder decomposition \cref{eq:simple-qr} with each exact prefix mass replaced by $V_t$ (\cref{eq:prefix-prob}).

\begin{table}[t]
  \centering
  \caption{
  \textbf{The TLM algorithm family}. \algBeam{} (\cref{fig:algos}) extends live particles to every live child before applying pruning.
  \algSWOR{} and \sworB{} (\cref{fig:swor}) extend \algBeam{} with pruning functions that sample distinct survivors without replacement.
  In contrast, \algSMC{} and \algSMCrb{} sample one child for each particle.
  \algSMC{} also samples the stop decision.
  \algSMCrb{} instead adds quotient and remainder contributions directly.
  The prune-free row is exact whenever the run halts.
  With $\algTopM$, keeping at most the $M$ heaviest particles, \algBeam{} is standard beam search (\cref{fig:algos}), and with $\pruneTau$ it is the beam summing of \citet{2025transducing}.
  The qualitative variance labels summarize the full-scale measurements of \cref{tab:variance-numbers}.\looseness=-1
    \looseness=-1
  }
  \label{tab:tlm-zoo}
  \small\setlength{\tabcolsep}{4pt}
  \begin{tabular}{@{}llcll@{}}
    \toprule
    Algorithm                              & Prune                             & (i) Unbiased & (ii) Variance (qual.) & (iii) Particle count  \\
    \midrule
    \algBeam{} (\cref{fig:algos})          & none                              & \cmark       & $0$                   & unbounded     \\
    \quad$\hookrightarrow$ $+$ $\algTopM$  & keep $M$ heaviest                 & \xmark       & $0$                   & fixed $M$     \\
    \quad$\hookrightarrow$ $+$ $\pruneTau$ & keep $\ge 1-\pruneThreshold$ mass & \xmark       & $0$                   & mass-set      \\
    \algSMC{} (\cref{app:smc})             & multinomial                       & \cmark       & high                  & fixed $M$     \\
    \algSMCrb{} (\cref{app:smc})           & multinomial                       & \cmark       & lower                 & fixed $M$     \\
    \algSWOR{}                              & SWOR                              & \cmark       & lowest                & fixed $M$     \\
    \sworB{}                                & SWOR, adaptive           & \cmark       & lowest                & adaptive, $m\le M$ \\
    \bottomrule
  \end{tabular}
\end{table}

\paragraph{What beam summing computes.}
Fix the target $\tgtStr$ and define the \defn{optimal twist} of a source prefix as the probability that a completion (\cref{sec:lms}) covers the target,
\begin{equation}
  \label{eq:optimal-twist}
  \twistOptimal(\srcStr)\defeq \Pr[\rvX\sim\pSource]{\tgtStr\preceq\transFn(\rvX) \mid \srcStr\preceq\rvX}
  = \expected{\srcStrpp\sim\pSource(\cdot\mid\srcStr)}{\indicator{\isMember(\srcStr\,\srcStrpp,\tgtStr)}}.
\end{equation}
It is optimal because weighting the source conditionals by it yields the source model conditioned on covering $\tgtStr$.
Under the corresponding importance weighting, every sampled path receives the same weight, so the target-prefix probability estimate has zero variance.
At $\prefixSource(\srcStr)=0$ we set $\twistOptimal(\srcStr)\defeq\indicator{\isMember(\srcStr,\tgtStr)}$, matching \cref{eq:tail-acceptance} under the conventions of \cref{sec:lms}.
The optimal twist returns the share of a prefix's mass that can still cover the target.\looseness=-1

The next lemma gives the backward recursion for the optimal twist and its subtree-mass identity.
The recursion mirrors prune-free \algBeam{}.
Its cylinder and stop terms are exactly the contributions accumulated by the algorithm, and unfolding the recursion from the empty prefix gives the quotient--remainder decomposition.
Write $\liveset{\srcStr}\defeq \{\srcSymp\in\srcAlphabet:\isLive(\srcStr\srcSymp,\tgtStr)\}$ for the \defn{live set} of next source symbols.
\begin{restatable}[A backward recursion for the optimal twist]{lemma}{lemTailEvidence}
  \label{lem:tail-evidence}
  The optimal twist $\twistOptimal$ of \cref{eq:optimal-twist} (relative to $\tgtStr$) satisfies the backward recursion
  \begin{equation}
    \label{eq:tail-acceptance}
    \twistOptimal(\srcStr)=
    \begin{cases}
      1                                                                      & \textbf{if } \isCylinder(\srcStr,\tgtStr) \\
      \indicator{\isMember(\srcStr,\tgtStr)}\,\prefixSource(\eos\mid\srcStr) & \textbf{otherwise}                        \\ \qquad +\displaystyle\smashoperator{\sum_{\srcSymp\in\liveset{\srcStr}}}\prefixSource(\srcSymp\mid\srcStr)\,\twistOptimal(\srcStr\srcSymp) &
    \end{cases}
  \end{equation}
  and every source prefix $\srcStr$ obeys
  \begin{equation}
    \label{eq:tail-subtree}
    \prefixSource(\srcStr)\,\twistOptimal(\srcStr)
    =\!\!\sum_{\substack{\srcStrp\in\preCover(\tgtStr)\\ \srcStr\preceq\srcStrp}}\!\!
    \pSource(\srcStrp),
  \end{equation}
  the source mass of the completions of $\srcStr$ that cover $\tgtStr$.
  In particular,
  $\twistOptimal(\varepsilon)=\prefixTarget(\tgtStr)$.
\end{restatable}
\begin{proof}
    The proof is deferred to \cref{sec:deferred-swor-proofs}.
\end{proof}

\begin{figure}[t]
  \centering
  \begin{minipage}[t]{0.49\linewidth}
    \begin{minted}[autogobble,xleftmargin=10pt]{python}
def @$\swor_{M,\twist}(\queuep)$@:
  @$\mathcal{W}\gets \{\weight^{(i)} \cdot \twist(\srcStr^{(i)})\colon (\srcStr^{(i)},\weight^{(i)}) \in \queuep \}$@
  @$\mathcal I_+ \gets \{i : \mathcal W^{(i)} > 0\}$@
  if @$|\mathcal I_+| \le M$@:
    return @$\{(\srcStr^{(i)},\weight^{(i)}) : i \in \mathcal I_+\}$@
  @$\pi \gets \algPPS(\mathcal{W}, M)$@
  @$\mathcal{S} \sim \algFixedWOR(\pi,\ M)$@
  return @$\{(\srcStr^{(i)}, \frac{\weight^{(i)}}{\pi_i}) : i \in \mathcal{S}\}$@
\end{minted}
    \begin{minted}[autogobble,xleftmargin=10pt]{python}
def @$\algPPS((\weight^{(i)})_{i=1}^{N},\ M)$@:
  @$\cert \gets \emptyset$@
  @$\pi \gets \mathbf{0}$@
  while True:
    @$\pi_i \gets \frac{(M - |\cert|)\,\weight^{(i)}}{\textstyle\sum_{j \notin \cert} \weight^{(j)}}$@  @$(i \notin \cert)$@ @\label{line:ppsassign}@
    @$\cert^{\prime} \gets \{i \notin \cert : \pi_i \ge 1\}$@ @\label{line:ppssat}@
    if @$\cert^{\prime} = \emptyset$@: break @\label{line:ppshalt}@
    @$\pi_{\cert^{\prime}} \gets 1$@
    @$\cert \gets \cert \cup \cert^{\prime}$@ @\label{line:ppspin}@
  return @$\pi$@
\end{minted}
    \begin{minted}[autogobble,xleftmargin=10pt]{python}
def @$\algFixedWOR(\pi,\ M)$@:
  @$\varphi \sim$@ uniform perm. of @$\{1, \ldots, N\}$@
  @$C_i \gets \textstyle\sum_{j \le i} \pi_{\varphi(j)}$@  @$(i = 0 \ldots N)$@
  @$u \sim \mathrm{Uniform}(0, 1)$@
  return @$\{\varphi(i) : \exists\, k < M,\ C_{i-1} \le u + k < C_i\}$@
\end{minted}

  \end{minipage}\hfill
  \begin{minipage}[t]{0.49\linewidth}

    \begin{minted}{python}
def @$\sworHalt_{M,\twist,\rho}(\queuep, \accum)$@:
  @$\mathcal{W}^{(i)} \gets \weight^{(i)} \cdot \twist(\srcStr^{(i)})$@
  @$W \gets \textstyle\sum_i \mathcal{W}^{(i)}$@
  @$\mathcal I_+ \gets \{i : \mathcal W^{(i)} > 0\}$@
  if @$W = 0$@:
    return @$\emptyset$@
  @$\mu \gets \textstyle\sum_i \min\!\big(1,\frac{\mathcal{W}^{(i)} M}{\rho\,(\accum + W)}\big)$@ @\label{line:sworbcounts}@
  @$m \gets \min(M,\ \lfloor\mu+\tfrac{1}{2}\rfloor)$@ @\label{line:sworbbudget}@
  if @$m = 0$@:
    @$B \sim \mathrm{Bernoulli}(\mu)$@ @\label{line:roulette}@
    if @$B = 0$@: return @$\emptyset$@
    @$m \gets 1$@,  @$\weight^{(i)} \gets \weight^{(i)}/\mu$@  @$(\forall i)$@
  if @$|\mathcal I_+| \le m$@:
    return @$\{(\srcStr^{(i)},\weight^{(i)}) : i \in \mathcal I_+\}$@
  @$\pi \gets \algPPS(\mathcal{W}, m)$@
  @$\mathcal{S} \sim \algFixedWOR(\pi,\ m)$@
  return @$\{(\srcStr^{(i)}, \frac{\weight^{(i)}}{\pi_i}) : i \in \mathcal{S}\}$@
\end{minted}
    \begin{minted}{python}
def @$\algSWOR(\tgtStr, M, \twist)$@:
  return @$\algBeam(\tgtStr, \twist, \swor_{M,\twist})$@
  
def @$\sworB(\tgtStr, M, \twist)$@:
  return
    @$\algBeam(\tgtStr, \twist, \sworHalt_{M,\twist,\rho})$@
\end{minted}
  \end{minipage}
  \caption{
    \textbf{SWOR.} The without-replacement resample \swor replaces the \prune in \cref{fig:algos}, and \sworHalt{} uses the mass accumulated so far to set the number of particles retained at each prune.
    To ensure halting even when $\accum=0$, this calculation uses the effective value $\accum_{\epsilon,t}$ of \cref{eq:effective-accum}.
    \Cref{sec:sworb} gives the full construction and its almost-sure halting guarantee.
    \algSWOR{} and \sworB{} are \algBeam{} run with these prunes.
  }
  \label{fig:swor}
\end{figure}

\subsection{Stochastic Pruning}
\label{sec:swor-estimator}
Running \algBeam{} with either of the prune functions in \cref{fig:algos} can give a biased estimate.
To make it unbiased, we introduce pruning functions that draw the survivors without replacement (SWOR) and correct for unequal selection using \defn{Horvitz--Thompson reweighting} \citep{horvitz1952generalization}.
This preserves the pool's sums in expectation.

SWOR retains every positive-weight particle when there are at most $M$ of them and otherwise replaces the pool by a size-$M$ subset $\mathcal S$, reweighting survivors by their inclusion probabilities \citep{fearnhead2003online}.
In \algBeam{}, particle $i$ in the pool passed to the prune is a child particle $\srcStr^{(m)}\srcSymp$ of its parent $\srcStr^{(m)}$, with weight $\weight^{(i)}=\weight^{(m)}\,\prefixSource(\srcSymp\mid\srcStr^{(m)})$ (\cref{fig:algos}).
For a function $h^{(i)}\defeq h(\srcStr^{(i)})$, SWOR estimates $\sum_{i=1}^{N}\weight^{(i)} h^{(i)}$ by
\begin{equation}
  \label{eq:ht}
  \swsum_h\;\defeq\;\sum_{i\in\mathcal S}\frac{\weight^{(i)}}{\pi_i}\,h^{(i)} .
\end{equation}
The estimate is unbiased for any valid inclusion probabilities, those with $\pi_i=0$ only where $\weight^{(i)}h^{(i)}=0$ \citep{horvitz1952generalization}.
With $I_i\defeq\indicator{i\in\mathcal S}$ and $\expected{\mathcal S}{I_i}=\pi_i$,
\begin{subequations}
  \begin{align}
    \expected{\mathcal S}{\swsum_h}  =\sum_{i=1}^{N}\expected{\mathcal S}{I_i}\,\frac{\weight^{(i)}}{\pi_i}\,h^{(i)}  =\sum_{i=1}^{N}\weight^{(i)}\,h^{(i)}.
  \end{align}
\end{subequations}

\paragraph{Choosing the inclusion probabilities.}
First consider the untwisted case, $\twist\equiv1$.
Write $\mathcal{I}_+$ for the set of positive-weight child particles.
When $|\mathcal{I}_+|\le M$ the prune keeps all of them unchanged.
Otherwise, any valid probabilities give an unbiased estimate, and we set each child particle's inclusion probability proportional to its weight, capped at one because a probability cannot exceed one,
\begin{equation}
  \label{eq:pps-marginals}
  \pi_i=\min\bigl(1,\,c\,\weight^{(i)}\bigr),\quad i=1,\ldots,N,
  \qquad
  \sum_{i=1}^{N}\min\bigl(1,\,c\,\weight^{(i)}\bigr)=M,
\end{equation}
with the constant $c>0$ fixed by the second equation.
The capped child particles form the \defn{certainty set} and survive with their weights unchanged \citep[\S3.2.5]{Shah2018}, and every non-certain survivor carries the same reweighted value $\weight^{(i)}/\pi_i$.
This means that the reweighted total matches the pool total on every draw, the randomness deciding which prefixes carry the mass forward.

The pruning function \swor{} (\cref{fig:swor}) combines these pieces.
If the pool has more than $M$ positive twisted weights, it hands them to \algPPS{} (probabilities proportional to size) for the inclusion probabilities $\pi_i\propto\weight^{(i)}\,\twist(\srcStr^{(i)})$ and their certainty set, and draws the size-$M$ subset with \algFixedWOR{} (\cref{fig:swor}).
Any fixed-size draw with those inclusion probabilities qualifies, and \algFixedWOR{}  realizes them exactly in $\bigO{N}$ \citep{madow1949theory}.
Each survivor gets the Horvitz--Thompson weight $\weight^{(i)}/\pi_i$ of \cref{eq:ht}.\looseness=-1

The recursion in \algPPS{} (\cref{fig:swor}) finds the certainty sets, pins child particles whose inclusion probabilities reach one, and recurses on the others with the remaining sample size.
Zero-weight child particles keep $\pi_i=0$ and are never drawn.
\algBeam{} with \swor{} as its pruning function is the particle filter of \citet{fearnhead2003online}, instantiated for the source-prefix enumeration over the precover.
We write \algSWOR{} for \algBeam{} run with the \swor{} prune (\cref{fig:swor}).

The second pruning function of \cref{fig:swor}, \sworHalt{}, uses the current-position running total to set the next live-particle count and save compute.
We introduce the resulting \sworB{} algorithm and its almost-sure halting guarantee in \cref{sec:sworb}, once the SWOR guarantees are in place (\cref{sec:swor-correctness}).\looseness=-1

\subsection{Unbiasedness Guarantees}
\label{sec:swor-correctness}
We now give sufficient conditions for when \algBeam{} (\cref{fig:algos}), combined with a \prune{} function, gives an unbiased estimate.
Without pruning, each particle carries its exact prefix mass (\cref{sec:two-algorithms}, \citealp{2025transducing}).
The lemma below shows that pruning preserves these pool weights in expectation.
The theorem that follows uses this result to show that, if the run halts, its accumulated quotient and remainder contributions sum to an unbiased estimate.

For a pool $\queue$, write $\weight_{\queue}(\srcStr)\defeq\sum_{(\srcStrp,\weight)\in\queue}\weight\,\idfunc\{\srcStrp=\srcStr\}$ for its weight at a source prefix $\srcStr\in\srcStrings$.
For every fixed choice of the inputs to the prune, let $\widetilde{\queue}\sim\pruneFunc(\queuep,\cdots)$ denote its random output.
We require
\begin{equation}
  \label{eq:prune-preserve}
  \expected{\widetilde{\queue}\sim\pruneFunc(\queuep,\cdots)}
  {\weight_{\widetilde{\queue}}(\srcStr)}
  =\weight_{\queuep}(\srcStr),
  \qquad\text{for every }\srcStr\in\srcStrings.
\end{equation}
The expectation in \cref{eq:prune-preserve} is only over the prune's internal randomness, with all its inputs, including $\queuep$ and $\accum$, held fixed.

For a run of \algBeam{} on a target string $\tgtStr$, consider the call for $\tgtStr$, at the top of the recursion.
Its starting pool $\queue_1$ holds the re-queued quotient entries of the target one symbol shorter, with total weight one at $\varepsilon$ when $\tgtStr=\varepsilon$.
For $i\ge1$, the extension step builds $\queue'_i$ from $\queue_i$, and the prune returns $\queue_{i+1}=\pruneFunc(\queue'_i,\cdots)$.
Write $\queueBar_i$ for the formal pools generated by the corresponding prune-free recursion, defined through their weight at each source prefix even if the complete reference enumeration does not halt.

\begin{restatable}[Pruning preserves the unpruned pool weights in expectation]{lemma}{lemPrunedProper}
  \label{lem:pruned-proper}
  Fix a target string $\tgtStr$ and run \algBeam{} (\cref{fig:algos}), with some prune function satisfying \cref{eq:prune-preserve}.
  Then, at every iteration $i$ and every source prefix $\srcStr$,
  \begin{equation}
    \expected{}{\weight_{\queue_i}(\srcStr)}
    =\weight_{\queueBar_i}(\srcStr).
  \end{equation}
\end{restatable}

\begin{proof}
    The proof is deferred to \cref{sec:deferred-swor-proofs}.
\end{proof}

\begin{restatable}[Unbiasedness]{theorem}{thmSworUnbiased}
  \label{thm:swor-unbiased}
  Fix a nonempty target string $\tgtStr$ and run \algBeam{} (\cref{fig:algos}) with a twist (\cref{sec:two-algorithms}).
  Suppose that
  \begin{enumerate}[label=(\arabic*),nosep]
    \item every prune preserves the pool's expected weight at every prefix (\cref{eq:prune-preserve}), and
    \item the run halts almost surely.
  \end{enumerate}
  Then $\mathbb{E}[\estZ_{|\tgtStr|}]=\prefixTarget(\tgtStr)$.
\end{restatable}

\begin{proof}
    The proof is deferred to \cref{sec:deferred-swor-proofs}.
\end{proof}

\begin{figure}[t]
  \centering
  \definecolor{accumgray}{RGB}{120,120,120}%
  \definecolor{pinteal}{RGB}{0,126,122}%
  \definecolor{dustorange}{RGB}{226,139,42}%
  \resizebox{\linewidth}{!}{%
\begin{tikzpicture}[font=\normalsize,
  bar/.style={draw=black!65, line width=0.4pt},
  pinned/.style={bar, fill=pinteal!60},
  dust/.style={bar, fill=dustorange!55},
  topup/.style={bar, fill=dustorange!20},
  accum/.style={bar, fill=accumgray!45},
  gone/.style={draw=black!40, dashed, line width=0.5pt},
  price/.style={black!75, dashed, line width=0.9pt},
  lbl/.style={font=\large}, leglbl/.style={font=\normalsize},
  tiny/.style={font=\normalsize, text=black!70}]

\def\bw{0.45}   
\def\gap{0.16}  
\def\c{1.2}     

\begin{scope}[shift={(0,0)}]
  \node[lbl, anchor=base] at (3.4, 3.35) {(a) Before the draw};
  \fill[accum] (0,0) rectangle (\bw,0.7);
  \node[tiny, anchor=south] at (0.5*\bw, 0.74) {$\accum$};
  \foreach [count=\i] \h in {2.6,1.9,1.5} {
    \fill[pinned] ({0.45+\i*(\bw+\gap)},0) rectangle ({0.45+\i*(\bw+\gap)+\bw},\h); }
  \foreach [count=\i] \h in {1.0,0.9,0.8,0.6,0.5,0.3} {
    \fill[dust] ({0.45+(\i+3)*(\bw+\gap)},0) rectangle ({0.45+(\i+3)*(\bw+\gap)+\bw},\h); }
  \draw[-{Stealth[length=2mm]}, black!60, line width=0.6pt] (0.82,0) -- (0.82,2.95) node[anchor=south, font=\normalsize, text=black!70] {$\mathcal{W}^{(i)}$};
  \draw[price] (0.95,\c) -- (6.55,\c) node[anchor=south east, font=\large, text=black!75, xshift=-2mm] {$1/c = \tfrac{\accum+W}{M}$};
\end{scope}

\draw[-{Stealth[length=2.5mm]}, black!70, line width=0.9pt] (6.85,0.9) -- (7.55,0.9);

\def\mx{7.95}   
\def\mb{0.42}   
\def\my{0.55}   
\node[lbl, anchor=base] at ({\mx+5*\mb}, 3.35) {(b) Count filled particles};
\fill[pinned] (\mx, \my) rectangle ({\mx+1*\mb}, \my+0.7);
\fill[pinned] ({\mx+1*\mb}, \my) rectangle ({\mx+2*\mb}, \my+0.7);
\fill[pinned] ({\mx+2*\mb}, \my) rectangle ({\mx+3*\mb}, \my+0.7);
\draw[gone, fill=pinteal!15] ({\mx+0*\mb}, \my+0.7) rectangle ({\mx+1*\mb}, \my+0.7+0.82);
\draw[gone, fill=pinteal!15] ({\mx+1*\mb}, \my+0.7) rectangle ({\mx+2*\mb}, \my+0.7+0.41);
\draw[gone, fill=pinteal!15] ({\mx+2*\mb}, \my+0.7) rectangle ({\mx+3*\mb}, \my+0.7+0.18);
\fill[dust] ({\mx+3*\mb}, \my) rectangle ({\mx+6.42*\mb}, \my+0.7);
\foreach \xseg in {3.83, 4.58, 5.25, 5.75, 6.17} {
  \draw[black!75, line width=0.5pt] ({\mx+\xseg*\mb}, \my) -- ({\mx+\xseg*\mb}, \my+0.7); }
\fill[bar, fill=dustorange!25] ({\mx+6*\mb}, \my) rectangle ({\mx+6.42*\mb}, \my+0.7);
\foreach \i in {0,...,8} { \draw[bar, fill=none] ({\mx+\i*\mb}, \my) rectangle ({\mx+(\i+1)*\mb}, \my+0.7); }
\draw[black!80, line width=1.2pt] ({\mx+6*\mb}, \my-0.14) -- ({\mx+6*\mb}, \my+0.84);
\node[tiny, anchor=south] at ({\mx+6*\mb}, \my+0.86) {$m$};
\node[tiny, anchor=south] at ({\mx+8.3*\mb}, \my+0.86) {$M{=}9$};
\node[tiny] at ({\mx+5*\mb}, \my-0.42) {$\mu = 6.42$,\quad $m=\lfloor\mu+\tfrac{1}{2}\rfloor=6$};

\draw[-{Stealth[length=2.5mm]}, black!70, line width=0.9pt] (12.55,0.9) -- (13.25,0.9);
\node[tiny, anchor=south] at (12.9, 1.05) {draw};

\begin{scope}[shift={(13.65,0)}]
  \node[lbl, anchor=base] at (3.4, 3.35) {(c) After the draw};
  \fill[accum] (0,0) rectangle (\bw,0.7);
  \node[tiny, anchor=south] at (0.5*\bw, 0.74) {$\accum$};
  \foreach [count=\i] \h in {2.6,1.9,1.5} {
    \fill[pinned] ({0.45+\i*(\bw+\gap)},0) rectangle ({0.45+\i*(\bw+\gap)+\bw},\h); }
  \fill[dust]  ({0.45+4*(\bw+\gap)},0) rectangle ({0.45+4*(\bw+\gap)+\bw},1.0);
  \fill[topup] ({0.45+4*(\bw+\gap)},1.0) rectangle ({0.45+4*(\bw+\gap)+\bw},1.367);
  \draw[gone]  ({0.45+5*(\bw+\gap)},0) rectangle ({0.45+5*(\bw+\gap)+\bw},0.9);
  \fill[dust]  ({0.45+6*(\bw+\gap)},0) rectangle ({0.45+6*(\bw+\gap)+\bw},0.8);
  \fill[topup] ({0.45+6*(\bw+\gap)},0.8) rectangle ({0.45+6*(\bw+\gap)+\bw},1.367);
  \draw[gone]  ({0.45+7*(\bw+\gap)},0) rectangle ({0.45+7*(\bw+\gap)+\bw},0.6);
  \fill[dust]  ({0.45+8*(\bw+\gap)},0) rectangle ({0.45+8*(\bw+\gap)+\bw},0.5);
  \fill[topup] ({0.45+8*(\bw+\gap)},0.5) rectangle ({0.45+8*(\bw+\gap)+\bw},1.367);
  \draw[gone]  ({0.45+9*(\bw+\gap)},0) rectangle ({0.45+9*(\bw+\gap)+\bw},0.3);
  \draw[black!45, dotted, line width=0.7pt] (0.95,\c) -- (6.55,\c) node[anchor=west, font=\normalsize, text=black!45] {$1/c$};
  \draw[price] (0.95,1.367) -- (6.55,1.367) node[anchor=south west, font=\large, text=black!75, xshift=-1mm] {$1/c'$};
\end{scope}

\def\ys{-3.8}
\def\ms{0.55}
\begin{scope}[shift={(2.15,\ys)}, xscale=\ms, yscale=\ms]
  \fill[accum] (0,0) rectangle (0.62,0.7);
  \foreach [count=\i] \h in {2.6,1.9,1.5} {
    \fill[pinned] ({\i*0.86},0) rectangle ({\i*0.86+0.62},\h); }
  \foreach [count=\i] \h in {1.0,0.9,0.8,0.6,0.5,0.3} {
    \fill[dust] ({(\i+3)*0.86},0) rectangle ({(\i+3)*0.86+0.62},\h); }
  \draw[price] (0.75,1.2) -- (8.4,1.2);
  \node[tiny] at (4.2,-0.85) {$m{=}6$};
\end{scope}
\draw[-{Stealth}, black!60] (7.25,\ys+1.0) -- (8.25,\ys+1.0);
\begin{scope}[shift={(8.65,\ys)}, xscale=\ms, yscale=\ms]
  \fill[accum] (0,0) rectangle (0.62,2.2);
  \foreach [count=\i] \h in {1.8,1.3} {
    \fill[pinned] ({\i*0.86},0) rectangle ({\i*0.86+0.62},\h); }
  \foreach [count=\i] \h in {0.8,0.55,0.4} {
    \fill[dust] ({(\i+2)*0.86},0) rectangle ({(\i+2)*0.86+0.62},\h); }
  \draw[price] (0.75,1.2) -- (5.2,1.2);
  \node[tiny] at (2.8,-0.85) {$m{=}3$};
\end{scope}
\draw[-{Stealth}, black!60] (11.85,\ys+1.0) -- (12.85,\ys+1.0);
\begin{scope}[shift={(13.25,\ys)}, xscale=\ms, yscale=\ms]
  \fill[accum] (0,0) rectangle (0.62,3.2);
  \foreach [count=\i] \h in {0.9,0.45,0.25} {
    \fill[dust] ({\i*0.86},0) rectangle ({\i*0.86+0.62},\h); }
  \draw[price] (0.75,1.2) -- (3.5,1.2);
  \node[tiny] at (1.9,-0.85) {$m{=}1$};
\end{scope}
\draw[-{Stealth}, black!60] (15.55,\ys+1.0) -- (16.55,\ys+1.0);
\begin{scope}[shift={(16.95,\ys)}, xscale=\ms, yscale=\ms]
  \fill[accum] (0,0) rectangle (0.62,3.8);
  \fill[dust] (0.86,0) rectangle (1.48,0.35);
  \draw[price] (0.75,1.2) -- (2.3,1.2);
  \node[font=\normalsize] at (2.0,0.75) {\Large\bf ?};
  \node[tiny] at (1.3,-0.85) {$m{=}0$};
\end{scope}

\def\yl{-5.4}
\fill[accum]  (1.5,\yl) rectangle +(0.42,0.28);  \node[leglbl, anchor=west] at (2.0,\yl+0.14) {accumulated mass $\accum$};
\fill[pinned] (6.1,\yl) rectangle +(0.42,0.28); \node[leglbl, anchor=west] at (6.6,\yl+0.14) {$\pi_i=1$, kept};
\fill[dust]   (11.8,\yl) rectangle +(0.42,0.19); \fill[topup] (11.8,\yl+0.19) rectangle +(0.42,0.14);
\node[leglbl, anchor=west] at (12.3,\yl+0.14) {drawn pps, $\pi_i=c'\,w^{(i)}$, raised to $1/c'$};
\draw[gone] (1.5,\yl-0.6) rectangle +(0.42,0.28); \node[leglbl, anchor=west] at (2.0,\yl-0.46) {dropped};
\draw[gone, fill=pinteal!15] (6.1,\yl-0.6) rectangle +(0.42,0.28); \node[leglbl, anchor=west] at (6.6,\yl-0.46) {excess above $1/c$};
\end{tikzpicture}}%
  \caption{One \sworHalt{} prune, left to right (top), and successive prunes at one target position (bottom).
  Each bar is one child particle of the pool with height its weight.
  The dashed line is the per-particle threshold $1/c=\rho\,(\accum+W)/M$ (drawn at $\rho{=}1$).
  Panel (b) shows how each child particle fills $\min(c\,w^{(i)},\,1)$ of one of the $M$ slots, so a child particle above the threshold counts for at most $1/c$ of the mass however heavy, the counts sum to $\mu$, and the particle count $m=\lfloor\mu+\tfrac{1}{2}\rfloor$, the nearest integer, is fixed before any randomness.
  In panel (c), \algPPS{} then keeps exactly $m$ particles, re-solving the constant to $c'$, pinning the child particles above the new threshold, and drawing each remaining survivor with probability $c'\,w^{(i)}$ at the shared weight $1/c'$, so the draw passes the twisted total through unchanged (\cref{prop:sworb}).
  Bottom row: as $\accum$ grows, the pool's bars shrink against the threshold and the particle count falls. When the pool counts for less than half a particle, a coin flip determines whether the call continues with a one-particle draw or halts.\looseness=-1
  }
  \label{fig:sworb-prune}
\end{figure}
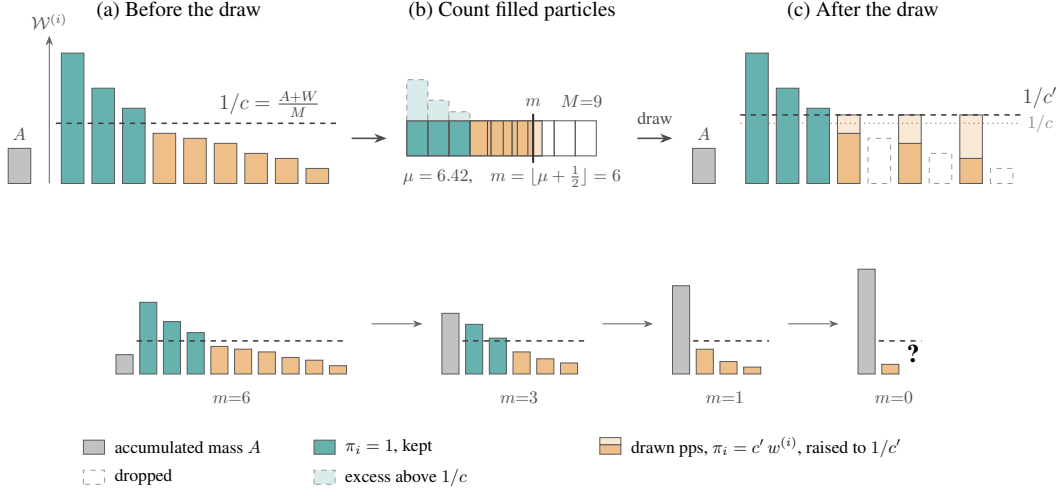

The stochastic pruning functions we consider (\cref{fig:swor}) meet the expected-weight hypothesis of \cref{thm:swor-unbiased}, while halting is established separately for each algorithm.
The draw \swor{} meets hypothesis (1) by the Horvitz--Thompson identity (\cref{eq:ht}), its inclusion probabilities are positive wherever the twisted weight is, and the extension step only hands it particles with positive twist (\cref{line:enumAlphabet}).
For hypothesis (2), the full \algSWOR{} run halts whenever the decomposition is finite at every recursive prefix position.
For target-prefix calls whose source-prefix enumeration would otherwise continue indefinitely, the tail-roulette prune of \cref{sec:halting} preserves hypothesis (1) and makes each call halt almost surely.
For \sworB{}, both hypotheses follow from \cref{prop:sworb}.

\begin{figure}[t!]
  \centering
  \begin{tikzpicture}[CenterFix, font=\scriptsize]
    \node[state, initial, initial text={}, accepting] (s) {$s$};
    \path (s) edge[loop above] node{\charfont{a}\,:\,$\varepsilon$} ();
    \path (s) edge[loop right] node{\charfont{b}\,:\,\charfont{c}} ();
  \end{tikzpicture}
  \caption{The one-state transducer of the halting example. Both arcs loop on the single initial, final state, $\charfont{a}$ emitting nothing and $\charfont{b}$ emitting $\charfont{c}$, so a source string maps to one $\charfont{c}$ for each $\charfont{b}$ it contains.}
  \label{fig:fst-loop}
\end{figure}
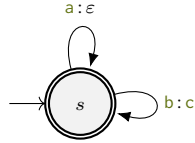

\subsection{Adaptive-Budget SWOR and Almost-Sure Halting}
\label{sec:sworb}
For a finite quotient--remainder decomposition, every branch of the \algBeam{} enumeration reaches a cylinder or a dead end (a prefix whose live set $\liveset{\srcStr}$ is empty) after finitely many iterations, so the call halts \citep[Thm.~4.1]{2025transducing}.
When the decomposition is infinite, this may not be the case.\looseness=-1

As an example, take the one-state transducer of \cref{fig:fst-loop} and the target $\tgtStr=\charfont{c}$.
Choose a source model that assigns positive probability to both $\charfont{a}$ and $\charfont{b}$ after every prefix $\charfont{a}^{k}$.
A prefix covers $\charfont{c}$ once it holds a $\charfont{b}$, and every extension keeps covering, so the decomposition is all quotient, $\Quotient(\charfont{c})=\{\charfont{a}^{k}\charfont{b} : k\ge0\}$, one cylinder at every depth.
At depth $k$, the prefix $\charfont{a}^{k}$ has two children, the cylinder child $\charfont{a}^{k}\charfont{b}$ and the prefix $\charfont{a}^{k+1}$. At the next iteration, $\charfont{a}^{k}\charfont{b}$ is recognized as a cylinder and its weight is added to $\accum$.
The prefix $\charfont{a}^{k+1}$ remains live, is never a member, a cylinder, or dead, and is extended in turn.
For any $M\ge2$, the prune keeps both child particles and no randomness enters, so the call never halts.
Although the weight along this path decays to zero, the source-prefix enumeration may continue forever.
To ensure halting, the fixed-$M$ methods of \cref{tab:tlm-zoo} compose their usual prune with the tail-roulette prune of \cref{sec:halting}.
The prune presented below instead halts this example on its own.\looseness=-1

The \swor{} function uses a fixed cap of $M$ particles at every iteration.
We instead decrease the live-particle count as the current-position running total grows, saving compute.
The result is the \sworHalt{} prune of \cref{fig:swor}, and \sworB{} is \algBeam{} run with it.
We call this method \defn{adaptive-budget SWOR}.
At every \pruneFunc{} call (\cref{line:sworbcounts} of \cref{fig:swor}), we set
\begin{equation}
  \label{eq:sworb-threshold}
  c = \frac{M}{\rho\cdot(\accum+W)}.
\end{equation}
Here $\accum$ is the current-position running total (\cref{fig:algos}), $W$ is the pool's total twisted weight, and $M$ is the maximum particle count.
The corresponding per-particle threshold $1/c$ sets the twisted-weight scale for the count $\mu$.
Thus $\rho\in(0,1]$ controls this count and, through $m$, the number of particles retained.
Children below the threshold contribute in proportion to their twisted weight, while those at or above it contribute one (the dashed line of \cref{fig:sworb-prune}(a,b)).
The resulting $\mu$ fixes the particle count $m$ before sampling according to
\begin{equation}
  \label{eq:sworb-budget}
  \mu=\sum_{i=1}^{N}\min\Bigl(1,\ c\weight^{(i)}\twist(\srcStr^{(i)})\Bigr),
  \qquad
  m=\min\bigl(M,\ \lfloor\mu+\tfrac{1}{2}\rfloor\bigr).
\end{equation}
This replaces the calculation of \cref{eq:pps-marginals}, where the constant is solved so that the capped counts sum to the fixed particle count $M$.
\algPPS{} then keeps exactly $m$ particles, re-solving the constant to $c'$, so every non-pinned survivor carries the shared weight $1/c'$ (\cref{fig:sworb-prune}(c)).
In the limit $\rho\to0$, \sworHalt{} recovers the fixed-$M$ \swor{} prune.
Under the live twist, the capped count satisfies $\mu\le (M/\rho)\,W/(\accum+W)$, where $W/(\accum+W)$ is the pool's remaining live share.
A source-extension iteration never increases the pool's mass.
When the pool counts for less than half a particle, $\mu<\tfrac{1}{2}$, a Russian roulette draw \citep{kahn1955} decides whether to empty the pool (\cref{line:roulette}).
To also cover a position that can continue forever without adding anything to $\accum$, we replace it at position $t$ by
\begin{equation}
  \label{eq:effective-accum}
  \accum_{\epsilon,t}\defeq\max\{\accum,\epsilon\estZ_{t-1}\},
  \qquad \estZ_0\defeq1,\quad \epsilon>0.
\end{equation}
This substitution applies only to the count and threshold in
\cref{line:sworbcounts,eq:sworb-threshold}.
The capped count $\mu$ therefore vanishes as the pool's total live weight $W$ vanishes relative to $\accum_{\epsilon,t}$.
The preceding estimate $\estZ_{t-1}$ is fixed before position $t$ begins.
The resulting particle count is therefore fixed before sampling, so the SWOR draw still preserves each weight in expectation.
\Cref{fig:sworb-prune} traces one draw through panels (a) to (c), as $\accum$ rises and the particle count $m$ falls.

\begin{figure}[t]
  \centering
  \tikzset{
    >=Stealth,
    gnode/.style={circle, draw=ETHGray!55, fill=ETHGray!7, line width=0.5pt, inner sep=0pt, minimum size=4.6mm},
    gedge/.style={ETHGray!45, line width=0.5pt},
    keptn/.style={circle, draw=jademist, fill=jademist, inner sep=0pt, minimum size=5mm},
    considn/.style={circle, draw=jademist!75, dashed, fill=ETHGray!7, line width=0.6pt, inner sep=0pt, minimum size=4.6mm},
    kedge/.style={jademist, line width=1.2pt},
    cedge/.style={jademist!70, line width=0.6pt, dashed},
    bdot/.style={circle, draw=slatelake, fill=slatelake!85, inner sep=0pt, minimum size=#1},
    bpath/.style={slatelake, line width=1.0pt},
    dup/.style={saffron, line width=1.3pt},
    dead/.style={ETHGray, line width=0.9pt, dashed, opacity=0.55},
    slbl/.style={font=\scriptsize\color{ETHGray}},
  }
  \begin{subfigure}[t]{0.49\linewidth}
    \centering
    \resizebox{\linewidth}{!}{\begin{tikzpicture}
        \useasboundingbox (-0.5,2.7) rectangle (6.8,7.3);
        \node[gnode](Br)  at (0,5)    {};
        \node[gnode](Bn1) at (2.5,6.4){}; \node[gnode](Bn2) at (2.5,5.0){}; \node[gnode](Bn3) at (2.5,3.6){};
        \node[gnode](Bl1) at (5,6.75) {}; \node[gnode](Bl2) at (5,6.05){};
        \node[gnode](Bl3) at (5,5.35) {}; \node[gnode](Bl4) at (5,4.65){};
        \node[gnode](Bl5) at (5,3.95) {}; \node[gnode](Bl6) at (5,3.25){};
        \draw[gedge] (Br)--(Bn1) (Br)--(Bn2) (Br)--(Bn3)
        (Bn1)--(Bl1) (Bn1)--(Bl2) (Bn2)--(Bl3) (Bn2)--(Bl4) (Bn3)--(Bl5) (Bn3)--(Bl6);
        \draw[cedge] (Br)--(Bn3) (Bn1)--(Bl2) (Bn2)--(Bl4);
        \node[considn] at (2.5,3.6){}; \node[considn] at (5,6.05){}; \node[considn] at (5,4.65){};
        \draw[kedge] (Br)--(Bn1) (Br)--(Bn2) (Bn1)--(Bl1) (Bn2)--(Bl3);
        \node[keptn] at (0,5){}; \node[keptn] at (2.5,6.4){}; \node[keptn] at (2.5,5.0){};
        \node[keptn] at (5,6.75){}; \node[keptn] at (5,5.35){};
        \draw[kedge,->] (5,6.75)--++(1.0,0);  \draw[kedge,->] (5,5.35)--++(1.0,0);
        \node[slbl, anchor=west] at (5.25,6.05){pruned};
        \node[slbl, anchor=south] at (2.5,6.95){extend};
      \end{tikzpicture}}
    \caption{All-child extension followed by SWOR pruning (\algSWOR{}, \sworB{}).}
    \label{fig:bvs-beam}
  \end{subfigure}
  \hfill
  \begin{subfigure}[t]{0.49\linewidth}
    \centering
    \resizebox{\linewidth}{!}{\begin{tikzpicture}
        \useasboundingbox (-0.5,-2.3) rectangle (6.8,2.3);
        \node[gnode](Sr)  at (0,0)    {};
        \node[gnode](Sn1) at (2.5,1.4){}; \node[gnode](Sn2) at (2.5,0.0){}; \node[gnode](Sn3) at (2.5,-1.4){};
        \node[gnode](Sl1) at (5,1.75) {}; \node[gnode](Sl2) at (5,1.05){};
        \node[gnode](Sl3) at (5,0.35) {}; \node[gnode](Sl4) at (5,-0.35){};
        \node[gnode](Sl5) at (5,-1.05){}; \node[gnode](Sl6) at (5,-1.75){};
        \draw[gedge] (Sr)--(Sn1) (Sr)--(Sn2) (Sr)--(Sn3)
        (Sn1)--(Sl1) (Sn1)--(Sl2) (Sn2)--(Sl3) (Sn2)--(Sl4) (Sn3)--(Sl5) (Sn3)--(Sl6);
        \draw[saffron!85, dashed, line width=0.7pt] (3.75,-2.0) -- (3.75,1.5);
        \node[slbl, anchor=south, text=saffron] at (3.75,1.85){$ESS<\threshold M$};
        \draw[bpath] (Sr)--(Sn1) (Sr)--(Sn2) (Sr)--(Sn3);
        \draw[bpath] (Sn1)--(Sl1);  \draw[dup] (Sn1)--(Sl2);
        \draw[bpath] (Sn2)--(Sl4);
        \draw[dead] (Sn3)--(3.62,-1.25);
        \node[bdot=4mm] at (0,0){};
        \node[bdot=5.6mm] at (2.5,1.4){};
        \node[bdot=4mm]   at (2.5,0.0){};
        \node[bdot=3mm]   at (2.5,-1.4){};
        \node[bdot=4mm] at (5,1.75){};
        \node[bdot=4mm, draw=saffron, fill=saffron] at (5,1.05){};
        \node[bdot=4mm] at (5,-0.35){};
        \node[ETHGray, font=\footnotesize] at (3.62,-1.25){$\times$};
        \node[slbl, anchor=north] at (3.55,-1.4){drop};
        \node[slbl, text=saffron, anchor=south] at (4.55,1.16){copy};
        \draw[bpath,->] (5,1.75)--++(1.0,0); \draw[bpath,->] (5,1.05)--++(1.0,0); \draw[bpath,->] (5,-0.35)--++(1.0,0);
        \foreach \x/\t in {0/{$t$}, 2.5/{$t{+}1$}, 5/{$t{+}2$}} { \node[slbl] at (\x,-2.25){\t}; }
      \end{tikzpicture}}
    \caption{Sampled extension followed by ESS-triggered multinomial resampling (\algSMC{}, \algSMCrb{}).}
    \label{fig:bvs-smc}
  \end{subfigure}
  \caption{
    Two ways to maintain a bounded particle pool.
    Dashed rings in (\subref{fig:bvs-beam}) mark children removed by pruning.
    In (\subref{fig:bvs-smc}), the cross marks a dropped particle and orange marks a duplicate.
    Grey edges show unexplored branches.
    Dot size shows weight.}
  \label{fig:beam-vs-smc}
\end{figure}
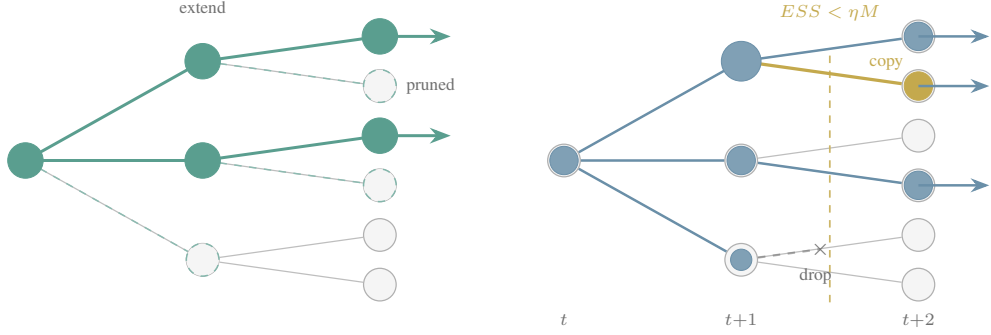

\begin{restatable}[Adaptive-budget SWOR]{proposition}{propSworb}
  \label{prop:sworb}
  Fix a positive integer $M$, a twist $\twist$, $\rho\in(0,1]$, and $\epsilon>0$, and run \sworB{} using $\accum_{\epsilon,t}$ from \cref{eq:effective-accum}.
  Then (I) each prune call satisfies \cref{eq:prune-preserve}.
  Under the live twist $\twist=\isLive$, (II) the full run halts almost surely.
  Consequently, under the live twist, $\expected{}{\estZ_t}=\prefixTarget(\tgtStr_{\le t})$ at every position.
\end{restatable}
\begin{proof}
    The proof is deferred to \cref{sec:sworb-proof}.
\end{proof}
\sworB{} uses $\accum_{\epsilon,t}$ and the current twisted weights to set its particle count, which can give a compute advantage at matched accuracy, as we experimentally show in \cref{fig:cost-pareto}.
An ablation of the $\rho$ parameter is given in \cref{app:knobs}.\looseness=-1

\begin{figure}[t!]
  \centering
  \includegraphics[width=\linewidth]{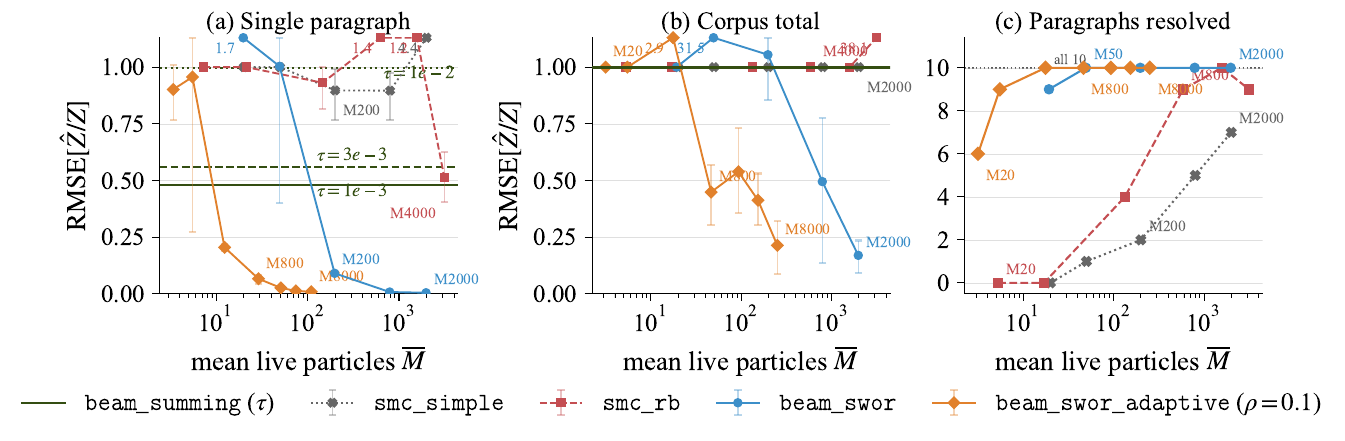}
  \caption{Unbiased methods versus beam summing with $\pruneTau$ on the $\pbigram{\circ}\STPTB$ TLM, compared with exact probabilities. Panels (a) and (b) plot relative root-mean-square error $\mathrm{RMSE}[\estZ/\evZ]$: (a)~a single example paragraph and (b)~the ten-paragraph corpus total, both against the mean live particle count. Labels give the maximum particle count $M$. Bars give $95\%$ bootstrap intervals. Panel (c)~counts the paragraphs each method resolves, capturing at least $5\%$ of a paragraph's mass.}
  \label{fig:validation}
\end{figure}

\subsection{SMC with Multinomial Resampling}
\label{sec:sampled-variants}

The \algSWOR{} estimator (\cref{fig:swor}) keeps the full enumeration of \algBeam{} and randomizes only the pruning.
A standard sequential Monte Carlo (SMC) construction instead uses multinomial resampling with replacement, gated by an \emph{effective sample size} test (ESS, \cref{app:smc}), as its prune (\cref{fig:beam-vs-smc}).
Two such algorithms complete the family of \cref{tab:tlm-zoo} and serve as the baselines of \cref{sec:experiments}.
First, \algSMC{} (\cref{fig:smc}) is a textbook particle filter \citep{smcbook,naesseth2019elements} run with the precover-restricted proposal.
It collects mass only when a particle happens to draw $\eos$ at a member prefix, and its resample duplicates heavy particles, so the pool's $M$ slots hold dependent repeats exactly where the weights have collapsed.
Second, \algSMCrb{} (\cref{fig:algos-smc}) marginalizes the stop branch rather than drawing it, a Rao--Blackwellization of the stop decision.
Both \algSMCrb{} and \algSWOR{} use importance weighting to account for the probability mass of children left out of the pool.
\algSMCrb{} samples one live child for each particle and may resample particles with replacement. 
\algSWOR{} instead forms every live child before pruning and samples distinct survivors without replacement.
\Cref{app:smc} develops both SMC baselines in full.
It proves \algSMCrb{} unbiased with the same nested induction as \cref{sec:swor-correctness}, after showing that its sampled extension gives each child the same expected weight as the all-child extension.\looseness=-1

\subsection{Relation to Prior Work}
\label{sec:swor-prior}
Our SWOR prune functions combine unequal-probability sampling without replacement with Horvitz--Thompson reweighting \citep{horvitz1952generalization}.
We implement the draw using systematic pps \citep{madow1949theory} and the certainty-set construction from without-replacement particle-filter resampling \citep{fearnhead2003online}.
Conditional Poisson stochastic beam search \citep{meister2021conditional} likewise samples prefixes without replacement, but its Horvitz--Thompson estimate of expectations over complete source strings requires their inclusion probabilities under the full search.
We instead correct each pruning draw locally, and these expected-weight identities compose into an unbiased estimate of the precover's total mass \citep{Shah2018}.
\citet{Shah2018} consider a fixed number of sampling stages, whereas a target-prefix call may require an unbounded number of source-extension iterations.
We extend the argument to this setting while adding quotient and remainder contributions, recursing over target positions, and adapting the particle count (\cref{thm:swor-unbiased,prop:sworb}).\looseness=-1

\section{Experiments}
\label{sec:experiments}
Unbiasedness, criterion (i) of \cref{sec:swor}, is established in \cref{sec:swor}.
We measure the bias of beam-summing estimates computed with $\pruneTau$ and the variance of the unbiased estimates on paragraphs from WikiText-2 \citep{merity2017pointer} in \cref{sec:ptb-lane}, confirming that variance decreases as the number of particles grows, criterion (ii).
\Cref{sec:cost} shows that running time can be controlled at the cost of variance, criterion (iii).\looseness=-1

On a DNA-to-amino-acid transduction whose precover grows exponentially, \cref{sec:dna} shows that \algBeam{} run with $\pruneTau$ gets prohibitively slow after a dozen amino acids, while \sworB{} reaches position $200$ in about five seconds.
Finally, \cref{sec:psycholing} revalidates a published psycholinguistic analysis \citep{kiegeland2025proper} by running beam summing with \sworHalt{} rather than the biased pruning function $\pruneTau$.\looseness=-1

\paragraph{Estimators.}
We evaluate the four unbiased methods \algSMC{}, \algSMCrb{}, \algSWOR{}, and \sworB{} (\cref{fig:swor} and \cref{sec:sampled-variants}), against the prior work's beam summing, \algBeam{} with the mass prune $\prune_{\pruneThreshold}$ (\cref{fig:algos}).
Every setting uses the live twist $\isLive$.
The comparison in \cref{sec:estimator-variance} separates the effect of marginalizing the $\eos$ decision rather than sampling it (\algSMC{} versus \algSMCrb{}) from the combined change to extending every live child and pruning without replacement (\algSMCrb{} versus \algSWOR{}).

\paragraph{Reporting.}
Over $S$ seeds we report the log of the averaged estimate, $\log\bigl(\tfrac{1}{S}\sum_{s=1}^{S}\estZ_s\bigr)$.
Each seed's $\estZ_s$ is unbiased for $\evZ$, so the average is unbiased as well, and its log estimates $\log\evZ$.
Averaging in the other order, $\tfrac{1}{S}\sum_{s}\log\estZ_s$, would estimate $\expected{}{\log\estZ}$ instead, and $\log$ is concave, so by Jensen's inequality $\expected{}{\log\estZ}\le\log\expected{}{\estZ}=\log\evZ$, a downward bias.
We report between-seed standard deviations alongside method comparisons.
We report particle counts as the mean number of live particles.\looseness=-1

\subsection{Byte-to-Word Transduction on Real Text}
\label{sec:ptb-lane}
The Penn Treebank (PTB) tokenizer segments raw text into PTB units, for example splitting \texttt{!!!} into three successive \texttt{!} units \citep{marcus-etal-1993-building}.
Following \citet{2025transducing,kiegeland2025proper}, we write $\STPTB$ for a transducer encoding these rules and evaluate it on ten WikiText-2 paragraphs.
We now measure what the bias correction changes in this setting using a bigram model.
The source model $\pbigram{}$ is a smoothed byte-level bigram model trained on the WikiText-2-raw train split and evaluated on the test paragraphs.\footnote{We use add-$\alpha$ smoothing with $\alpha=0.5$ \citep[Ch.~3]{jm3}.}

We calculate the exact target prefix probabilities by composing $\STPTB$ with the weighted finite-state representation of $\pbigram{}$. Standard forward and backward path sums in the resulting automaton give the per-position values, as described in \cref{app:tailacc-solve}.

Results are shown in \Cref{fig:validation}, using the relative root-mean-square error across seeds, $\mathrm{RMSE}[\estZ/\evZ]=\sqrt{S^{-1}\sum_{s=1}^{S}(\estZ_s/\evZ-1)^2}$, as a function of the mean live particle count.
On the single paragraph, threshold-pruned beam summing has error floors from $0.48$ at the lowest threshold to $1.0$ at the harshest.
On the corpus, its error remains close to $1$.
The with-replacement samplers are unbiased, but their corpus error remains near $1$ even with thousands of particles.
\algSWOR{} and \sworB{} drive the corpus error down to $0.17$ and $0.21$ at their largest tested $M$ values, between-seed standard deviations $57\times$ and $49\times$ below \algSMC{}'s (\cref{tab:variance-numbers}).
Panel~(c) shows they resolve most paragraphs where \algSMCrb{} needs $M$ in the thousands to do so and \algSMC{} never resolves the full corpus.
\Cref{app:perpos-grid} shows the comparison at every position of all ten paragraphs.
The unbiased methods track the exact ratio closely, while \algBeam{} with $\pruneTau$ falls below it when pruning discards covering mass.\looseness=-1

\begin{figure}[tpb]
  \centering
  \includegraphics[width=\linewidth]{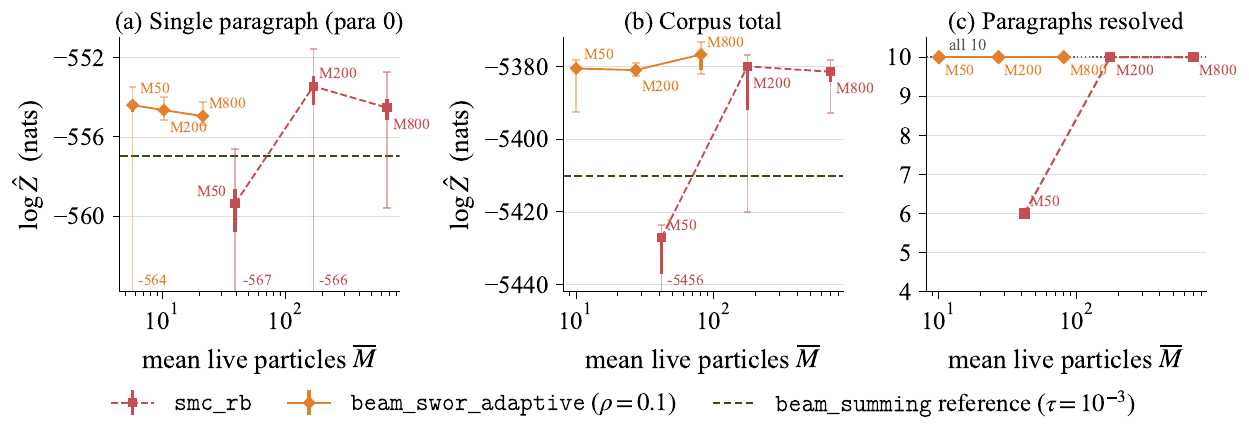}
  \caption{GPT-2-large transduced with $\STPTB$.
  Panels (a) and (b) plot the estimates $\log \estZ$ in nats against the mean live particle count.
  Labels give the maximum particle count $M$.
  Each point is $\log\bigl(8^{-1}\sum_{s=1}^{8}\estZ_s\bigr)$, with a failed run contributing $\estZ_s=0$.
  Thick bars give $95\%$ bootstrap intervals, thin whiskers span the finite per-run log estimates, and clipped minima are printed at the lower edge.
  The dashed line is the threshold-pruned lower bound at $\pruneThreshold{=}10^{-3}$.
  The two algorithms appear to converge to a common value. The lower bound is $2.0$ nats below this value on the paragraph and approximately $33$ nats below it on the corpus.
  Panel (c) counts the paragraphs on which each method recovers at least the lower bound's mass.\looseness=-1}
  \label{fig:gpt2-ptb}
\end{figure}
\paragraph{GPT-2 on PTB.}
We also run $\STPTB$ with GPT-2 large, $\gpt$ \citep{radford2019language}.
Exact target prefix probabilities are unavailable for GPT-2 on the WikiText paragraphs, so we report between-seed variance and convergence as $M$ grows and compare with the threshold-pruned lower bound.
The estimates in \cref{fig:gpt2-ptb} are shown for \algSMCrb{}, the lowest-variance method with replacement, and \sworB{}, the lowest-variance method without replacement (\cref{tab:variance-numbers}), along with the lower bound at $\pruneThreshold=10^{-3}$.
At the largest tested value of $M$, the two agree within $0.5$ nats on the paragraph. On the corpus, \algSMCrb{} remains $4.7$ nats below the \sworB{} plateau, which is reached at a much smaller mean live particle count.
\Cref{tab:gpt2-head-to-head} reports estimates for each of the ten paragraphs. The threshold-pruned estimate lies below the \sworB{} estimate on every paragraph, by $0.24$ to $7.7$ nats and $33.4$ nats in total. The between-seed standard deviation of \sworB{} is smaller than this gap on every paragraph and at least an order of magnitude smaller on five of ten.\looseness=-1

\begin{figure}[tpb]
  \centering
  \begin{subfigure}[t]{0.49\linewidth}
    \centering
    \includegraphics[width=\linewidth]{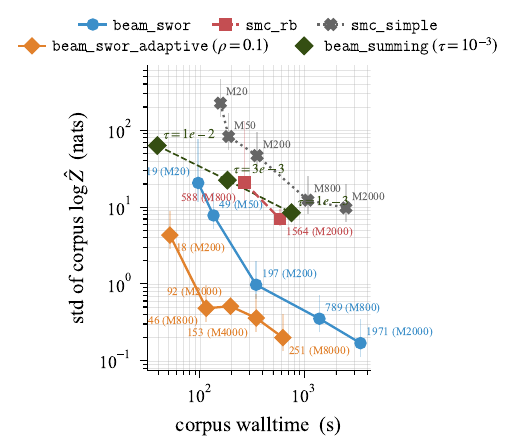}
    \caption{$\pbigram{}$}
    \label{fig:cost-pareto-bigram}
  \end{subfigure}
  \hfill
  \begin{subfigure}[t]{0.49\linewidth}
    \centering
    \includegraphics[width=\linewidth]{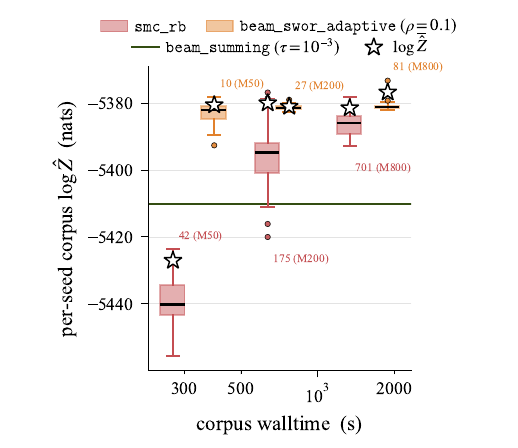}
    \caption{$\gpt$}
    \label{fig:cost-pareto-gpt2}
  \end{subfigure}
  \caption{
    Accuracy--walltime comparisons.
    (a)~For $\pbigram{}$, where exact target prefix probabilities are available, the between-seed standard deviation of the corpus log-estimate against per-seed corpus walltime. Annotations in both panels give the mean live particle count, with the maximum particle count $M$ in parentheses. The \algBeam{} lines using $\pruneTau$ report bias relative to the exact values, and bars give the 95\% confidence interval of each standard deviation.
    (b)~For $\gpt$, where exact target prefix probabilities are unavailable, the finite per-seed corpus log-estimates over $32$ seeds. Boxes span the interquartile range with whiskers at $1.5$ IQR and seeds beyond them drawn as dots, and the black bar marks the median seed. The star marks the pooled estimate $\log \overline{\hat Z}$, the log of the seed-averaged $\hat Z$, computed from all $32$ seeds with a failed run contributing $\hat Z = 0$. Because the pooled estimate averages $\hat Z$ before taking the logarithm, occasional high estimates can make it larger than the median per-seed estimate. The \algBeam{} line shows its corpus total at $\pruneTau$.
  }
  \label{fig:cost-pareto}
\end{figure}

\subsection{Comparing Compute--Variance Frontiers}
\label{sec:cost}
We compare the estimators against walltime in \cref{fig:cost-pareto}.\footnote{All timings come from one workstation with an Intel Core i9-12900KF, 128\,GB of RAM, and a single NVIDIA RTX 3090 with 24\,GB of memory. For the CPU experiments, we report the total CPU time of each run, pinned to a P-core on an otherwise idle machine. For GPT-2 we report wall-clock time with one job at a time on the GPU.}
With the $\pbigram$ source, there is no GPU overhead, and with $\gpt$ the walltime includes the source model's forward passes.
On $\gpt\circ\STPTB$, \sworB{} at nominal $M{=}200$ takes $772$ seconds, compared with $1224$ seconds for \algBeam{} with $\pruneTau$. The $M{=}800$ point of \cref{tab:gpt2-head-to-head} takes $1.5\times$ as long as that baseline.
Under both source models, the \sworB{} frontier dominates both SMC baselines.
On the bigram, at effectively matched between-seed standard deviation ($0.36$ against $0.35$ nats), \sworB{} is four times faster than \algSWOR{} ($346$ against $1390$ seconds) and uses fewer live particles ($153$ against $789$). At matched walltime, its standard deviation is nearly three times lower ($0.36$ against $0.98$ nats).
With $\rho{=}0.1$ at nominal $M{=}2000$, $149$ seeds give a standard deviation of $0.163$ nats ($95\%$ CI $[0.146,0.184]$), consistent with the $0.170$ nats of fixed-$M$ \algSWOR{} at about one fifth of its CPU time (\cref{app:knobs}).
The parameter $\rho$ of \sworB{} is swept in \cref{app:knobs}.\looseness=-1

\begin{figure}[t]
  \centering
  \begin{subfigure}[b]{0.58\linewidth}
    \centering
    \includegraphics[width=\linewidth]{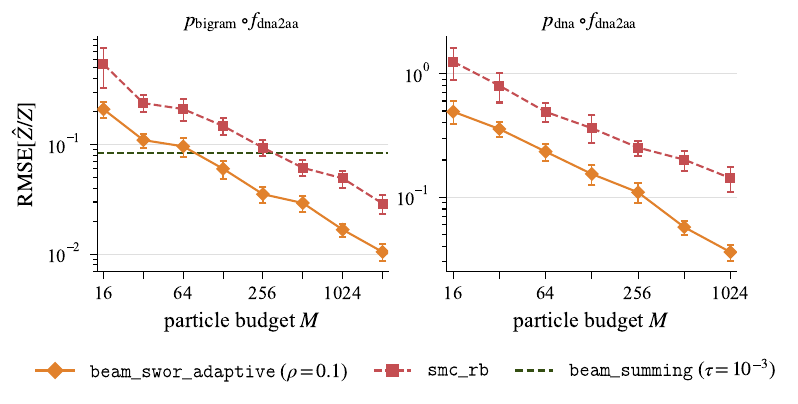}
    \caption{Validation against exact values.}
    \label{fig:dna-validation}
  \end{subfigure}
  \hfill
  \begin{subfigure}[b]{0.39\linewidth}
    \centering
    \includegraphics[width=\linewidth]{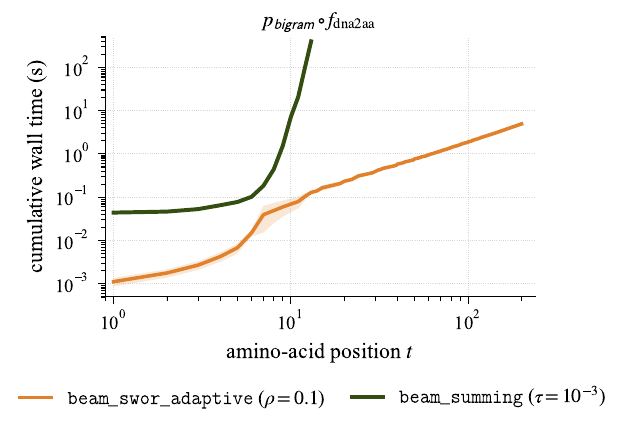}
    \caption{Cumulative walltime per amino acid.}
    \label{fig:dna-scaling}
  \end{subfigure}
  \caption{
    DNA-to-amino-acid results.
    (\subref{fig:dna-validation})~Relative RMSE against exact values over $64$ seeds and three targets, with $95\%$ bootstrap intervals.
    The error of \sworB{} decreases faster with $M$ than that of \algSMCrb{}, while \algBeam{} with $\pruneTau$ has fixed error.
    (\subref{fig:dna-scaling})~Cumulative walltime for the bigram model.
    \algBeam{} with $\pruneTau$ exceeds five minutes at position $13$, while \sworB{} ($\rho{=}0.1$, $M{=}256$) reaches position $200$ in about five seconds.
  }
\end{figure}

\subsection{Transducing DNA Language Models}
\label{sec:dna}
The DNA$\to$amino-acid map \transFnDNA{} maps DNA strings over a four-letter alphabet to amino-acid strings. Combined with a high-entropy source model, it gives the setting where materializing the precover is most costly in our experiments.
Each amino acid is the image of several source codons, and the high-entropy model places non-negligible mass on all of them, so the set of source strings an exhaustive method must track multiplies for each added symbol.
We train a bigram model on the DNA training corpus, and use $\pdna$, a GPT-2 model trained on human DNA, the same model used by \citet{2025transducing}.
\Cref{fig:dna-validation} compares both DNA models with exact target prefix probabilities obtained by brute-force enumeration of the precover.
In both the validation and scaling comparisons, \algBeam{} uses $\pruneTau$ with $\pruneThreshold=10^{-3}$.

\Cref{fig:dna-scaling} shows the running time of \algBeam{} with $\pruneTau$ growing rapidly within a dozen amino acids, while \sworB{} scales approximately linearly.
\Cref{app:worstcase} gives a uniform-source worst case for \algBeam{} with $\algTopM$. Its captured mass ratio vanishes exponentially in the target length, while \algSMC{}, \algSMCrb{}, and fixed-$M$ \algSWOR{} are exact for any $M$.

\subsection{Variance Across Estimators}
\label{sec:estimator-variance}
\Cref{tab:variance-numbers} compares the four unbiased estimators across the three transductions.
Marginalizing the \eos{} decision rather than sampling it (\algSMC{} versus \algSMCrb{}) reduces variance only modestly.
The large reduction in every setting occurs between \algSMCrb{} and \algSWOR{}, which changes both the extension to every live child and the pruning to sampling without replacement.
Compared to \algSWOR{}, \sworB{} achieves comparable variance with less computation on the bigram corpus, lower variance on GPT-2, and identical results on DNA, where \sworHalt{} retains the same number of particles as \swor{} (\cref{tab:variance-numbers,fig:cost-pareto}).

\begin{table}[ptb]
  \centering
  \caption{
  Between-seed standard deviation, in nats, of each estimator's log probability estimate $\log\estZ$.
  The columns report the ten-paragraph WikiText-2 total for the bigram model, the first WikiText-2 paragraph for GPT-2, and the amino-acid target for DNA.
  Reported values use eight seeds and maximum particle counts of $2000$, $200$, and $256$, respectively, except \sworB{} on the bigram ($M{=}8000$, mean live count $251$).
  The \algSMC{} GPT-2 run covered only $120/850$ bytes.}
  \label{tab:variance-numbers}
  \small\setlength{\tabcolsep}{4pt}
  \begin{tabular}{@{}lccc@{}}
    \toprule
    Algorithm               & $\pbigram{\circ}\STPTB$ & $\gpt{\circ}\STPTB$ & $\pbigram{\circ}\STDNA$ \\
    \midrule
    \algSMC  & 9.78                    & ---    & 0.111 \\
    \algSMCrb & 7.86                    & 3.23                & 0.107 \\
    \algSWOR                & 0.170                   & 0.395               & 0.062 \\
    \sworB{} ($\rho{=}0.1$) & 0.200                   & 0.199               & 0.062 \\
    \bottomrule
  \end{tabular}
\end{table}

\subsection{Psycholinguistics: Surprisal for Reading-Time Prediction}
\label{sec:psycholing}
Surprisal theory hypothesizes that processing time increases with contextual surprisal, the negative log probability of a linguistic unit in context.
\citet{kiegeland2025proper} test this hypothesis on the MECO eye-tracking corpus \citep{siegelman2022expanding}, computing contextual surprisal from GPT-2 Small through transducers for PTB words, whitespace-delimited words, and characters.
The surprisals come from the per-position beam-summing estimate of \citet{2025transducing}, computed with $\pruneTau$.
Each next-unit probability is estimated by the ratio $\estZ_t/\estZ_{t-1}$.
Pruning lowers both estimates, but it can remove different fractions of their mass, so the ratio and resulting surprisal can err in either direction.\looseness=-1

Beam summing with $\pruneTau$ gives no way to determine how the resulting bias affects these conclusions.
We therefore replace $\pruneTau$ with unbiased SWOR pruning, recompute the PTB-word contextual surprisals, and rerun the corresponding fits.
The resulting estimate of corpus surprisal is $106$ nats lower than the original.
We evaluate first-fixation duration, gaze duration, and total reading time.
These are, respectively, the duration of the first fixation on a word, the sum of its first-pass fixations, and the sum of all its fixations including refixations.
For each measure, $\dll$ is the per-observation improvement in held-out log likelihood from adding contextual surprisal to the shared baseline.\looseness=-1

\begin{table}[tpb]
  \centering
  \small
  \caption{
  Predictive contribution of PTB-word contextual surprisal on MECO.
  Values are $\dll$ in $10^{-3}$ nats. For \sworHalt{}, we use $\rho{=}0.1$, nominal $M{=}8000$, and $10$ seeds.
    $^{*}$\,$\significance < 0.05$; $^{**}$\,$\significance < 0.01$; $^{***}$\,$\significance < 0.001$.
}
  \label{tab:gamm_smc_within}
  \begin{tabular}{@{}lrrr@{}}
    \toprule
    Pruning function  & First fixation  & Gaze duration                    & Total reading time \\
    \midrule
    \sworHalt{} & $0.76^{**}$     & $2.05^{***}$                     & $3.26^{***}$  \\
    $\pruneTau$   & $0.81^{**}$     & $2.13^{***}$                     & $3.24^{***}$  \\
    \bottomrule
  \end{tabular}
\end{table}

For each reading-time measure, we compare fits using contextual surprisal computed by beam summing with either $\pruneTau$ or \sworHalt{}. The model specification, evaluation protocol, and beam-based unigram-surprisal control of \citet{kiegeland2025proper} are unchanged (\cref{tab:gamm_smc_within}; details in \cref{app:meco-anatomy}).
We find no evidence of a predictive difference between the two fits on any reading-time measure (one-sided $\significance\ge 0.068$ in both directions).
Under either estimator, adding contextual surprisal computed over PTB words improves held-out prediction for all three reading-time measures (\cref{tab:gamm_smc_within}).
Thus, replacing $\pruneTau$ with \sworHalt{} in beam summing leaves the published conclusions unchanged.\looseness=-1

\section{Conclusion}

In the beam-summing method of \cite{2025transducing}, the threshold prune is the only approximation, discarding mass to control the computation.
This work replaces that prune with a weight-proportional draw without replacement and Horvitz--Thompson reweighting while retaining the enumeration and the per-position estimate.
\sworB{} uses the current-position running total to set its particle count.
Its roulette rule ensures that the run halts almost surely.
The resulting estimator is unbiased and replaces unknown pruning error with sampling variance.
The method applies to target prefix probability estimation for a pretrained source model composed with a deterministic transformation, providing an unbiased reference for estimating threshold-pruning error.
In the DNA$\to$amino-acid setting, it also makes target prefix probability estimation feasible for longer targets over high-entropy sources.

\section*{Acknowledgements}
VS is supported by DNRF grant P1 and NNF grant NNF26OC0114089.

\bibliographystyle{iclr2027_conference}
\bibliography{example_paper}

\makeatletter
\let\@currlist\@empty  
\makeatother
\clearpage
\appendix
\onecolumn

\let\addcontentsline\realaddcontentsline
\etocdepthtag.toc{appendix}
\etocsettagdepth{mainmatter}{none}
\etocsettagdepth{appendix}{subsection}
\etocsetnexttocdepth{subsection}
\etocsettocstyle
{\section*{Contents of the Appendix}}
{}
\tableofcontents

\newpage

\section{Notation Glossary}
\label{sec:glossary}

\begin{center}
  \small
  \begin{tabular}{l|p{0.62\textwidth}}
    \toprule
    Notation                                            & Gloss                                                                                                                                                                                                                              \\
    \midrule
    $\varepsilon$                                       & Empty string                                                                                                                                                                                                                       \\
    $\eos$                                              & End-of-string symbol                                                                                                                                                                                                               \\
    $\srcSym,\srcSymp \in \srcAlphabet$                 & Symbols in the source alphabet $\srcAlphabet$                                                                                                                                                                                      \\
    $\srcStr,\srcStrp \in \srcStrings$                  & Source strings                                                                                                                                                                                                                      \\
    $\srcStrings$                                       & Set of all source strings                                                                                                                                                                                                          \\
    $\srcStr\srcStrp$                                   & Concatenation                                                                                                                                                                                                                      \\
    $\tgtSym \in \tgtAlphabet$                          & Symbols in the target alphabet $\tgtAlphabet$                                                                                                                                                                                      \\
    $\tgtStr \in \tgtStrings$                           & Target strings, with $\tgtStr_{\le t}$ the length-$t$ target prefix                                                                                                                                                                \\
    $\tgtStrings$                                       & Set of all target strings                                                                                                                                                                                                          \\
    $\transFn\colon\srcStrings\to\tgtStrings$           & Functional string-to-string transformation                                                                                                                                                                                         \\
    $\transST$                                          & Transducer implementing $\transFn$                                                                                                                                                                                                 \\
    \midrule
    $\srcStr\preceq\srcStrp$                            & $\srcStr$ is a prefix of $\srcStrp$                                                                                                                                                                                                \\
    $\basisSet{\srcStr}$                                & Cylinder spanned by $\srcStr$, i.e.\ $\srcStr\srcStrings$                                                                                                                                                                          \\
    $\pr(\srcStringsSub)$                               & Prefix-free operation on $\srcStringsSub$ (\cref{sec:lms})                                                                                                                                                                         \\
    $\preCover(\tgtStr)$                                & Precover of $\tgtStr$: $\{\srcStr:\tgtStr\preceq\transFn(\srcStr)\}$ (\cref{eq:simple-qr})                                                                                                                                         \\
    $\maxQuotient(\tgtStr)$                             & Largest union of cylinders contained in $\preCover(\tgtStr)$                                                                                                                                                                       \\
    $\Quotient(\tgtStr)$                                & Quotient: prefix-free generators of $\maxQuotient(\tgtStr)$                                                                                                                                                                        \\
    $\Remainder(\tgtStr)$                               & Remainder: $\preCover(\tgtStr)\setminus\maxQuotient(\tgtStr)$                                                                                                                                                                      \\
    \midrule
    $\pSource$                                          & Source language model over $\srcStrings$                                                                                                                                                                                           \\
    $\prefixSource(\srcStr)$                            & Source prefix probability                                                                                                                                                                                                          \\
    $\prefixSource(\srcSymp\mid\srcStr)$                & Conditional next-symbol prefix probability                                                                                                                                                                                         \\
    $\prefixSource(\eos\mid\srcStr)$                    & Source stop probability                                                                                                                                                                                                            \\
    $\pTarget$                                          & Transduced (target) language model                                                                                                                                                                                                 \\
    $\prefixTarget(\tgtStr)$                            & Target prefix probability, the normalizing constant $\evZ$ we estimate (\cref{eq:evidence})                                                                                                                                        \\
    \midrule
    $\isMember(\srcStr,\tgtStr)$                        & $\srcStr\in\preCover(\tgtStr)$: the output emitted covers $\tgtStr$                                                                                                                                                                \\
    $\isCylinder(\srcStr,\tgtStr)$                      & $\basisSet{\srcStr}\subseteq\preCover(\tgtStr)$: every extension covers $\tgtStr$                                                                                                                                                  \\
    $\isLive(\srcStr,\tgtStr)$                          & Some extension of $\srcStr$ covers $\tgtStr$                                                                                                                                                                                       \\
    $\liveset{\srcStr}$                                 & Live set: live next source symbols $\{\srcSymp\in\srcAlphabet:\isLive(\srcStr\srcSymp,\tgtStr)\}$                                                                                                                                 \\
    \midrule
    $M$                                                 & Fixed or maximum particle count                                                                                                                                                                                                    \\
    $\{(\srcStr^{(m)},\weight^{(m)})\}_{m=1}^{M}$       & Particle pool: source prefixes and their weights                                                                                                                                                                                   \\
    $\weight^{(m)}$                                     & Particle weight                                                                                                                                                                                                                     \\
    $\weight_0^{(m)}$                                   & Reference weight inherited by particle $m$ under tail roulette (\cref{sec:halting})                                                                                                                                              \\
    $\estZ$                                             & Prefix-probability estimate                                                                                                                                                                                                        \\
    $\twistOptimal(\srcStr),\ \twistOptimal_t(\srcStr)$ & Optimal twist of $\srcStr$, relative to $\tgtStr$ or $\tgtStr_{\le t}$                                                                                                                         \\
    $\liveZ_t(\srcStr)$                                 & Twisted live mass, $\sum_{\srcSymp\in\liveset{\srcStr}}\prefixSource(\srcSymp\mid\srcStr)\twist_t(\srcStr\srcSymp)$                                                                              \\
    $\propZ_t(\srcStr)$                                 & Proposal normalizer, $\indicator{\isMember(\srcStr,\tgtStr_{\le t})}\,\prefixSource(\eos\mid\srcStr)+\sum_{\srcSymp'\in\liveset{\srcStr}}\prefixSource(\srcSymp'\mid\srcStr)\,\twist_t(\srcStr\srcSymp')$   \\
    $\mathrm{ESS},\ \threshold$                         & Effective sample size (\cref{eq:ess}), resample when $\mathrm{ESS}<\threshold M$                                                                                                                                                   \\
    $\pruneThreshold$                                   & Prior work's prune threshold, keep $\ge1-\pruneThreshold$ of the pool's mass                                                                                                                                   \\
    $\accum$                                           & Current-position running total, returned as $\estZ_t$                                                                                                                                   \\
    $\epsilon$                                         & Halting constant  \\
    $\rho$                                              & Scales \sworHalt's threshold $1/c$                                                                                                                                                          \\
    $\kappa$                                            & Tail-roulette threshold relative to $\weight_0$ (\cref{sec:halting})                                                                                                                       \\
    \midrule
    $\cert$                                             & Certainty set: child particles pinned to $\pi_i=1$                                                                                                                                                                                 \\
    $\mathcal{S}$                                       & Sampled child set, of size $M$ or $m$                                                                                                                                                                                              \\
    $\weight^{(i)}/\pi_i$                               & Horvitz--Thompson reweighting                                                                                                                                                                                                      \\
    \bottomrule
  \end{tabular}
\end{center}

\paragraph{Notation conventions.}
\textbf{Color} encodes domain: {\color{TokenColor}magenta} for source-domain objects ($\srcStr$, $\srcAlphabet$, $\preCover$, $\Quotient$, $\Remainder$) and {\color{CharacterColor}olive} for target-domain objects ($\tgtStr$, $\tgtAlphabet$, $\transFn$).
\textbf{Font} encodes type: lowercase italic for symbols ($\srcSym$, $\tgtSym$), bold italic for strings ($\srcStr$, $\tgtStr$), calligraphic for sets and alphabets ($\srcAlphabet$, $\tgtAlphabet$, $\preCover$, $\Quotient$, $\Remainder$), and typewriter for the transducer ($\transST$) and for algorithm and predicate names ($\isMember$, $\isCylinder$, $\isLive$, $\algSMC$, $\algResample$).
The overrightarrow $\prefixLanguageOp{\cdot}$ marks prefix probabilities ($\prefixSource$ is the prefix probability of $\pSource$).

\begin{figure}[ptb]
  \centering
  \small
  \begin{minipage}[t]{0.48\linewidth}
    \begin{minted}[autogobble,xleftmargin=10pt]{python}
def @$\algSMC(\prefixProposal, \potential, M)$@:
  @$\mathcal{R} \gets \emptyset$@
  @$\queue \gets ([\varepsilon, \tfrac{1}{M}])_{m=1}^{M}$@ @\label{line:smcinit}@
  while @$|\queue| > 0$@:
    @$\queuep \gets \emptyset$@
    for @$(\srcStr, \weight) \in \queue$@:
      @$\textbf{\texttt{sample }}\srcSymp \sim \prefixProposal(\cdot \mid \srcStr)$@ @\label{line:smcdraw}@
      @$\weight \gets \weight \cdot \prefixSource(\srcSymp \mid \srcStr)\,/\,\prefixProposal(\srcSymp \mid \srcStr)$@ @\label{line:smcweight}@
      if @$\srcSymp = \eos$@:
        @$\mathcal{R}$@.add(@$(\srcStr,\ \weight \cdot \potential(\srcStr))$@) @\label{line:smceos}@
      else:
        @$\queuep$@.add(@$(\srcStr\srcSymp,\ \weight)$@)
    @$\queue \gets \algResample(\queuep)$@
  return @$\sum_{(\srcStr, \weight) \in \mathcal{R}} \weight$@ @\label{line:smcreturn}@
\end{minted}
  \end{minipage}
  \begin{minipage}[t]{0.45\linewidth}
    \begin{minted}[autogobble,xleftmargin=10pt]{python}
def @$\algResample(\queue)$@:
  if @$|\queue| = 0$@: return @$\emptyset$@
  if @$\mathrm{ESS}(\queue) \ge \threshold\,|\queue|$@: @\label{line:smcesstest}@
    return @$\queue$@
  @$j_1, \ldots, j_{|\queue|} \simIID \mathrm{Categorical}\bigl(\weight^{(i)}/\textstyle\sum_i \weight^{(i)}\bigr)$@ @\label{line:smcredraw}@
  return @$\{(\srcStr^{(j_k)},\ \textstyle\sum_i \weight^{(i)}/|\queue|)\}_{k=1}^{|\queue|}$@ @\label{line:smcreset}@
\end{minted}
  \end{minipage}
  \caption{A minimal \textbf{sequential Monte Carlo} algorithm for estimating $\evZ$ (\cref{eq:evidence}).
    $M$ particles extend one symbol at a time under the proposal, with the incremental weights of \cref{eq:sis-weight-update}, finish with the potential factor at $\eos$, and hand uneven weights to \algResample{}. \algResample{} is an ESS-gated multinomial resample (\cref{eq:ess}) that samples with replacement and resets weights to their mean.
  }
  \label{fig:smc}
\end{figure}

\section{Sequential Monte Carlo over the Precover}
\label{app:smc}

This appendix first gives the proofs deferred from the main text (\cref{sec:deferred-swor-proofs,sec:sworb-proof}).
It then develops the with-replacement SMC algorithms for estimating the total mass of the precover, including the generic estimator and its resampling (\cref{sec:resampling}), the intermediate targets and the \algSMCrb{} algorithm (\cref{sec:intermediate-targets}), and the unbiasedness of \algSMCrb{} (\cref{sec:smc-single-rb}).
Finally, it gives the common tail-roulette prune for the fixed-$M$ methods (\cref{sec:halting}).

\subsection{Proofs for the SWOR Estimator}
\label{sec:deferred-swor-proofs}

\lemTailEvidence*
\begin{proof}
  Suppose first that $\prefixSource(\srcStr)>0$.
  If $\isCylinder(\srcStr,\tgtStr)$ holds, every completion of
  $\srcStr$ covers $\tgtStr$, so $\twistOptimal(\srcStr)=1$.
  Otherwise, condition on whether the completion stops immediately with
  $\eos$ or begins with a source symbol $\srcSymp$.
  The stop event covers $\tgtStr$ exactly when
  $\isMember(\srcStr,\tgtStr)$ holds, and after drawing $\srcSymp$ the
  conditional probability of covering $\tgtStr$ is
  $\twistOptimal(\srcStr\srcSymp)$.
  A non-live child has optimal twist zero, so the sum may be restricted
  to $\liveset{\srcStr}$, giving \cref{eq:tail-acceptance}.

  When $\prefixSource(\srcStr)=0$, the recursion follows from the
  conventions of \cref{sec:lms} and the definition of $\twistOptimal$ when $\prefixSource(\srcStr)=0$.
  If $\srcStr$ is a cylinder, then it is a member and both sides of
  \cref{eq:tail-acceptance} equal one.
  Otherwise, both sides equal
  $\indicator{\isMember(\srcStr,\tgtStr)}$.

  It remains to prove \cref{eq:tail-subtree}.
  For $\prefixSource(\srcStr)>0$, the definition of the optimal twist gives
  \begin{align}
    \prefixSource(\srcStr)\,\twistOptimal(\srcStr)
     & =
    \prefixSource(\srcStr)
    \sum_{\srcStrpp\in\srcStrings}
    \frac{\pSource(\srcStr\srcStrpp)}
    {\prefixSource(\srcStr)}
    \indicator{\isMember(\srcStr\srcStrpp,\tgtStr)}
    =
    \sum_{\substack{\srcStrp\in\preCover(\tgtStr) \\
        \srcStr\preceq\srcStrp}}
    \pSource(\srcStrp).
  \end{align}
  If $\prefixSource(\srcStr)=0$, every extension
  $\srcStrp\succeq\srcStr$ has $\pSource(\srcStrp)=0$, so both sides of
  \cref{eq:tail-subtree} are zero.

  Finally, setting $\srcStr=\varepsilon$ in
  \cref{eq:tail-subtree} gives
  \begin{equation}
    \twistOptimal(\varepsilon)
    =
    \sum_{\srcStrp\in\preCover(\tgtStr)}
    \pSource(\srcStrp)
    =
    \prefixTarget(\tgtStr),
  \end{equation}
  by \cref{eq:tlm-prefix-prob}.
\end{proof}

\lemPrunedProper*
\begin{proof}
  We induct on $N=|\tgtStr|$.
  For each $N$, we first establish the claim for the initial pools at $i=1$, then induct on the iterations within that call.

  \emph{Initial pools.} For $N=0$, $\tgtStr=\varepsilon$, both runs start with aggregate weight one at the empty prefix and zero at every other prefix.
  Now let $N>0$ and assume the result for the recursive call on $\tgtStr_{<N}$.
  The pool $\queue_1$ is exactly the quotient returned by the recursive call on $\tgtStr_{<N}$, while $\queueBar_1$ contains the corresponding prune-free quotient weights.
  Fix a source prefix $\srcStr$. Because the starting pool is prefix-free, $\srcStr$ can extend at most one of its entries.
  Since each iteration appends one source symbol, $\srcStr$ can be added to the call's quotient at most once.  
  The outer induction hypothesis therefore implies that that quotient has the same expected weight at every source prefix as its prune-free counterpart, so
  \begin{equation}
    \expected{}{\weight_{\queue_1}(\srcStr)} = \weight_{\queueBar_1}(\srcStr).
  \end{equation}
  Together with the case $N=0$, this establishes the within-call induction base $i=1$ for every $N$.

  \emph{Within-call induction.} Fix $N$ and assume the identity holds at iteration $i$ for every source prefix.
  The full extension step obeys
  \begin{equation}
    \label{eq:extension}
    \weight_{\queue'_i}(\srcStr\srcSymp)
    =
    \weight_{\queue_i}(\srcStr)\,
    \idfunc\{\neg\isCylinder(\srcStr,\tgtStr)\}\,
    \prefixSource(\srcSymp\mid\srcStr)\,
    \idfunc\{\twist(\srcStr\srcSymp)>0\}.
  \end{equation}
  A cylinder prefix produces no children, while every other particle produces one child for each symbol admitted by the twist.
  For every prefix $\srcStr$ and symbol $\srcSymp$, we therefore have
  \begin{subequations}
    \begin{align}
      \expected{}{\weight_{\queue_{i+1}}(\srcStr\srcSymp)}
       & =
      \expected{}{\weight_{\queue'_i}(\srcStr\srcSymp)}
      \justification{condition on the prune inputs and apply \cref{eq:prune-preserve}} \\
       & =
      \expected{}{\weight_{\queue_i}(\srcStr)}\,
      \idfunc\{\neg\isCylinder(\srcStr,\tgtStr)\}\,
      \prefixSource(\srcSymp\mid\srcStr)\,
      \idfunc\{\twist(\srcStr\srcSymp)>0\}
      \justification{\cref{eq:extension}}                                              \\
       & =
      \weight_{\queueBar_i}(\srcStr)\,
      \idfunc\{\neg\isCylinder(\srcStr,\tgtStr)\}\,
      \prefixSource(\srcSymp\mid\srcStr)\,
      \idfunc\{\twist(\srcStr\srcSymp)>0\}
      \justification{induction hypothesis}                                             \\
       & =
      \weight_{\queueBar'_i}(\srcStr\srcSymp)
      \justification{\cref{eq:extension} in the reference run}                         \\
       & =
      \weight_{\queueBar_{i+1}}(\srcStr\srcSymp).
      \justification{$\queueBar_{i+1}=\queueBar'_i$ under the identity prune}
    \end{align}
  \end{subequations}
  Since each iteration appends one source symbol, every prefix in either next pool is nonempty and has the form $\srcStr\srcSymp$. This proves the claim at iteration $i+1$ for every source prefix.
  Thus the within-call induction proves the claim at every iteration, completing the outer induction on $N$.
\end{proof}

\thmSworUnbiased*
\begin{proof}
  By hypothesis (2), every recursive call returns almost surely, so the contributions $V_t$ below and the final estimator $\estZ_{|\tgtStr|}$ are well defined.
  For each $1\le t\le |\tgtStr|$ and each source prefix $\srcStr$, let $\VBar_t(\srcStr)$ be the contribution assigned to $\srcStr$ by the deterministic prune-free reference enumeration at position $t$, as in \cref{eq:V-def}.
  The twist only gates the extension step in \algBeam{}.
  By definition it gates off no child source prefix with positive covering mass and does not alter particle weights.
  The exactness of prune-free \algBeam{} (\cref{sec:two-algorithms}) therefore gives $\VBar_t(\srcStr)=\prefixSource(\srcStr)$ for $\srcStr\in\Quotient(\tgtStr_{\le t})$, while $\VBar_t(\srcStr)\prefixSource(\eos\mid\srcStr)=\pSource(\srcStr)$ for $\srcStr\in\Remainder(\tgtStr_{\le t})$.

  We first show that pruning preserves these contributions in expectation, $\expected{}{V_t(\srcStr)}=\VBar_t(\srcStr)$ for every $1\le t\le |\tgtStr|$ and every $\srcStr\in\Quotient(\tgtStr_{\le t})\sqcup\Remainder(\tgtStr_{\le t})$, by induction on $t$.
  Fix such a position $t$ and an element $\srcStr$ of its decomposition.
  Since $\preCover(\tgtStr_{\le t})\subseteq\preCover(\tgtStr_{\le t-1})$, the element also belongs to the preceding precover.
  The preceding quotient--remainder split gives two alternatives.
  Either $\srcStr$ extends a unique $\srcStrp\in\Quotient(\tgtStr_{\le t-1})$, or $\srcStr\in\Remainder(\tgtStr_{\le t-1})$ (\cref{sec:tlms}).
  When $t=1$, $\Quotient(\varepsilon)=\{\varepsilon\}$ and $\Remainder(\varepsilon)=\emptyset$, so only the quotient alternative occurs.
  When $t>1$, assume the identity holds at position $t-1$.
  If $\srcStr$ is a remainder element at position $t$ and $\prefixSource(\eos\mid\srcStr)=0$, the convention in \cref{eq:V-def} gives $V_t(\srcStr)=\VBar_t(\srcStr)=0$.
  Thus, if $\srcStr$ is a remainder element, only the case $\prefixSource(\eos\mid\srcStr)>0$ remains.

  \emph{Extended from the quotient.} Let $\srcStrp$ denote the quotient prefix in the first alternative.
  The call at $t$ places the preceding quotient entries at $\srcStrp$ in $\queue_1$ and appends one source symbol per iteration. Thus any entries at $\srcStr$ can occur only at iteration $i=|\srcStr|-|\srcStrp|+1$.
  At that iteration, \cref{eq:V-def} identifies the contribution assigned to $\srcStr$ with its aggregate pool weight, so $V_t(\srcStr)=\weight_{\queue_i}(\srcStr)$ and $\VBar_t(\srcStr)=\weight_{\queueBar_i}(\srcStr)$.
  Applying \cref{lem:pruned-proper} gives
  \begin{equation*}
    \expected{}{V_t(\srcStr)}
    =\expected{}{\weight_{\queue_i}(\srcStr)}
    =\weight_{\queueBar_i}(\srcStr)
    =\VBar_t(\srcStr).
  \end{equation*}

  \emph{Retained from the remainder.} In the second alternative, the target re-check passes every entry at $\srcStr$ unchanged because $\srcStr\in\preCover(\tgtStr_{\le t})$ (\cref{line:targetstep}).
  Hence $V_t(\srcStr)=V_{t-1}(\srcStr)$ and $\VBar_t(\srcStr)=\VBar_{t-1}(\srcStr)$.
  The induction hypothesis therefore gives
  \begin{equation*}
    \expected{}{V_t(\srcStr)}
    =\expected{}{V_{t-1}(\srcStr)}
    =\VBar_{t-1}(\srcStr)
    =\VBar_t(\srcStr).
  \end{equation*}

  The zero-EOS-probability observation and the two alternatives complete the induction. Set $T=|\tgtStr|$. Then
  \begin{subequations}
    \begin{align}
      \mathbb{E}[\estZ_T]
       & =\!\!\sum_{\srcStr\in\Remainder(\tgtStr)}\!\!\mathbb{E}[V_T(\srcStr)]\,\prefixSource(\eos\mid\srcStr)\;+\!\!\sum_{\srcStr\in\Quotient(\tgtStr)}\!\!\mathbb{E}[V_T(\srcStr)]
      \justification{\cref{eq:pos-estimate} and Tonelli}
      \label{eq:assembly-swap}                                                                                                                                          \\
       & =\!\!\sum_{\srcStr\in\Remainder(\tgtStr)}\!\!\VBar_T(\srcStr)\,\prefixSource(\eos\mid\srcStr)\;+\!\!\sum_{\srcStr\in\Quotient(\tgtStr)}\!\!\VBar_T(\srcStr)
      \justification{the preservation above}                                                                                                                                         \\
       & =\!\!\sum_{\srcStr\in\Remainder(\tgtStr)}\!\!\pSource(\srcStr)\;+\!\!\sum_{\srcStr\in\Quotient(\tgtStr)}\!\!\prefixSource(\srcStr)
      \justification{the prune-free identities above}                                                                                                                                \\
       & =\prefixTarget(\tgtStr).
      \justification{\cref{eq:simple-qr}}
      \label{eq:assembly-means}
    \end{align}
  \end{subequations}
\end{proof}

\subsection{Proof for Adaptive-Budget SWOR}
\label{sec:sworb-proof}
By \cref{thm:swor-unbiased}, it is enough to show that every call to the adaptive-budget prune preserves the expected weight at each source prefix (\cref{eq:prune-preserve}) and that the full run halts almost surely.

\propSworb*
\begin{proof}
  (I) To establish \cref{eq:prune-preserve}, consider a single prune call with $\queuep$ and $\accum_{\epsilon,t}$ fixed.
  For particle $i$, write $\mathcal W^{(i)}=\weight^{(i)}\twist(\srcStr^{(i)})$ and $\mathcal I_+=\{i:\mathcal W^{(i)}>0\}$, as in \cref{fig:swor}.
  These inputs determine $W$, $\mu$, and the rounded value $m$ before the prune draws any randomness.
  So they determine which of the three cases below applies, even though the roulette case may subsequently change $m$ from zero to one.
  We consider the following three cases.
  \begin{enumerate}[label=(\arabic*),leftmargin=*]
    \item Suppose $W=0$.
    The extension step ensures $\twist(\srcStr^{(i)})>0$ for every particle in $\queuep$ (\cref{line:enumAlphabet}).
    Since $W=\sum_i\mathcal W^{(i)}=0$ and $\mathcal W^{(i)}\ge0$, every $\mathcal W^{(i)}=0$.
    Because $\twist(\srcStr^{(i)})>0$, this implies $\weight^{(i)}=0$ for every particle.
    Hence $\weight_{\queuep}(\srcStr)=\weight_{\emptyset}(\srcStr)=0$ for every $\srcStr$, which establishes \cref{eq:prune-preserve}.

    \item Suppose $W>0$ and $m\ge1$.
    We split according to whether $|\mathcal I_+|\le m$ or $|\mathcal I_+|>m$.
    If $|\mathcal I_+|\le m$, the prune keeps all particles in $\mathcal I_+$ unchanged.
    For every $i\notin\mathcal I_+$, $\mathcal W^{(i)}=0$.
    Since $\twist(\srcStr^{(i)})>0$, this implies $\weight^{(i)}=0$.
    Removing those particles therefore leaves the weight at every source prefix unchanged, so \cref{eq:prune-preserve} holds.
    If $|\mathcal I_+|>m$, the prune must select exactly $m$ particles.
    It computes their inclusion probabilities from the twisted weights with \algPPS{}, draws the particles, and applies Horvitz--Thompson reweighting.
    Here $\pi_i>0$ for every $i\in\mathcal I_+$.
    Taking $h^{(i)}=\indicator{\srcStr^{(i)}=\srcStr}$ in \cref{eq:ht} gives
    \begin{equation*}
      \mathbb{E}\left[\weight_{\widetilde{\queue}}(\srcStr)\right]
      =\weight_{\queuep}(\srcStr),
    \end{equation*}
    which is \cref{eq:prune-preserve}.

    \item Suppose $W>0$ and $m=0$.
    Because $W>0$, the definition of $\mu$ in \cref{eq:sworb-budget} gives $\mu>0$.
    Because $m=0$, the rounding rule in the same equation gives $\mu<\tfrac12$.
    Thus the division by $\mu$ below is well defined.
    Returning the empty pool deterministically would discard positive weight and violate \cref{eq:prune-preserve}, so the prune instead uses this same $\mu$ as its roulette survival probability.
    It draws $B\sim\mathrm{Bernoulli}(\mu)$ and returns the empty pool if $B=0$.
    If $B=1$, it sets $m=1$ and replaces every $\weight^{(i)}$ by $\weight^{(i)}/\mu$ before performing a one-particle SWOR prune.
    The resulting twisted weights are $\mathcal W^{(i)}/\mu$.
    This common rescaling leaves $\mathcal I_+$ unchanged and, if \algPPS{} is called, leaves its inclusion probabilities $\pi_i$ unchanged.
    The $m\ge1$ case above therefore gives
    \begin{equation*}
      \mathbb{E}\left[\weight_{\widetilde{\queue}}(\srcStr)\mid B=1\right]
      =\frac{\weight_{\queuep}(\srcStr)}{\mu}.
    \end{equation*}
    Averaging over $B$ gives
    \begin{align*}
      \mathbb{E}\left[\weight_{\widetilde{\queue}}(\srcStr)\right]
       &=\mu\,\frac{\weight_{\queuep}(\srcStr)}{\mu}+(1-\mu)0  =\weight_{\queuep}(\srcStr).
    \end{align*}
    Thus \cref{eq:prune-preserve} also holds in the third case.
  \end{enumerate}
  Together, the three cases establish (I).
  Multiplying \cref{eq:prune-preserve} by $\twist(\srcStr)$ and summing over source prefixes gives
  \begin{equation*}
    \mathbb{E}\left[\sum_{\srcStr}\weight_{\widetilde{\queue}}(\srcStr)\twist(\srcStr)\right]
    =\sum_{\srcStr}\weight_{\queuep}(\srcStr)\twist(\srcStr)
    =W.
  \end{equation*}
  This is the form needed in (II).

  (II) We prove by induction on target-prefix length that every call made by \sworB{} halts almost surely.
  The empty-prefix call halts after one iteration.
  Fix a nonempty target prefix $\tgtStr$, let $t=|\tgtStr|$, and suppose the recursive call for $\tgtStr_{<t}$ has returned.
  Its quotient is a finite starting pool, which we write as
  \begin{equation*}
    \queue_1
    =
    \{(\srcStr^{(j)},\weight^{(j)})\}_{j=1}^{J}.
  \end{equation*}
  It is enough to prove halting conditional on this pool and on the preceding estimate $\estZ_{t-1}$, so we hold both fixed below.

  If $\queue_1$ has no positive weight, the call halts immediately when the pool is empty or no later than its first prune by case (1) of (I).
  We may therefore assume that $\queue_1$ contains a positive weight.
  When $t=1$, $\estZ_0=1$.
  When $t>1$, every positive entry in $\queue_1$ was added to the quotient of the preceding call, and its weight was included in $\estZ_{t-1}$.
  Hence $\estZ_{t-1}>0$ and, throughout the current call,
  \begin{equation}
    \label{eq:sworb-positive-accum}
    \accum_{\epsilon,t}
    \ge
    \epsilon\estZ_{t-1}
    >0.
  \end{equation}

  For $n\ge1$, when the $n^{\text{th}}$ prune is reached, let $\queue'_n$ be the child pool passed to it and define
  $W_n\defeq  \sum_{\srcStr}  \weight_{\queue'_n}(\srcStr)\twist(\srcStr)$.
  If the current target-prefix call halts before this prune, define $W_n=0$.
  Let $\mu_n$ and $m_n$ denote the capped count and its rounded value before roulette at the $n^{\text{th}}$ prune, defining both to be zero if that prune is not reached or if $W_n=0$.

  Part (I) preserves the expected weight at every source prefix through each adaptive prune.
  By (I) and iterated expectation, a length-$n$ descendant $\srcStr^{(j)}\srcStrp$ has expected weight at most $\weight^{(j)}\prefixSource(\srcStrp\mid\srcStr^{(j)})$.
  Summing over the live descendants gives
  \begin{align*}
    \mathbb{E}\left[W_n\right]
    &\le
    \sum_{j=1}^{J}
    \weight^{(j)}
    \sum_{\srcStrp\in\srcAlphabet^n}
    \prefixSource(\srcStrp\mid\srcStr^{(j)})
    \xrightarrow{n\to\infty}0.
  \end{align*}
  The inner sum is the conditional probability of generating at least $n$ additional source symbols from $\srcStr^{(j)}$.
  It tends to zero because $\pSource$ is a distribution over finite strings (\cref{sec:lms}), and the outer sum has finitely many terms.

  At the $n^{\text{th}}$ prune, write $\mathcal W_n^{(i)}$ for its twisted particle weights, so $W_n=\sum_i\mathcal W_n^{(i)}$.
  Then \cref{eq:sworb-budget,eq:sworb-positive-accum} and $\min(1,x)\le x$ give
  \begin{align*}
    \mu_n
    &=
    \sum_i
    \min\left(
      1,
      \frac{M\mathcal W_n^{(i)}}{\rho\,(\accum_{\epsilon,t}+W_n)}
    \right) \le
    \frac{M}{\rho\,(\accum_{\epsilon,t}+W_n)}
    \sum_i\mathcal W_n^{(i)}  =
    \frac{M W_n}{\rho\,(\accum_{\epsilon,t}+W_n)}  \le
    \frac{M W_n}{\rho\,\epsilon\estZ_{t-1}}.
  \end{align*}
  If this prune is not reached, the last inequality still holds because $\mu_n=W_n=0$ by definition.
  Taking expectations therefore gives
  \begin{equation*}
    \mathbb{E}\left[\mu_n\right]
    \le
    \frac{M}{\rho\,\epsilon\estZ_{t-1}}
    \mathbb{E}\left[W_n\right]
    \xrightarrow{n\to\infty}0.
  \end{equation*}
  The rounding rule in \cref{eq:sworb-budget} gives $m_n>0$ exactly when $\mu_n\ge\tfrac12$.
  Markov's inequality now gives
  \begin{equation}
    \label{eq:sworb-positive-count}
    \Pr{m_n>0}
    =
    \Pr{\mu_n\ge\tfrac12}
    \le
    2\mathbb{E}\left[\mu_n\right]
    \xrightarrow{n\to\infty}0.
  \end{equation}

  Let $E_n$ be the event that the call survives the $n^{\text{th}}$ prune.
  Write $p_n\defeq\Pr{E_n}$ and $p_0=1$.
  The events $E_n$ are decreasing, so $p_n$ converges to some $p\ge0$.
  If $m_n=0$ and $W_n>0$, roulette lets the pool survive with probability $\mu_n<\tfrac12$.
  If $W_n=0$, the prune returns the empty pool.
  Conditioning on the state before the prune therefore gives
  \begin{equation*}
    p_n
    \le
    \Pr{m_n>0}
    +\tfrac12 p_{n-1}.
  \end{equation*}
  Taking $n\to\infty$ and using \cref{eq:sworb-positive-count} gives $p\le\tfrac12p$, so $p=0$.
  The events $E_n$ are decreasing, and their intersection is the event that the call never halts.
  Therefore
  \begin{equation*}
    \Pr{\bigcap_{n\ge1}E_n}
    =
    \lim_{n\to\infty}p_n
    =0.
  \end{equation*}
  Thus the current call halts almost surely for every realized finite starting pool.
  The induction proves that the full run halts almost surely.
  Applying \cref{thm:swor-unbiased} to each target prefix using (I) and (II) gives the final claim.\looseness=-1
\end{proof}

\subsection{Sequential Monte Carlo}
\label{sec:resampling}
The sum we wish to estimate (\cref{eq:tlm-prefix-prob}) is over source strings filtered by a membership test in the precover $\preCover(\tgtStr)$.
A \defn{potential} is a function $\potential\colon\srcStrings\to\mathbb{R}_{\ge0}$ that weights each source string \citep{smcbook}, and $\evZ$ is its expectation under the source model,
\begin{equation}\label{eq:evidence}
  \evZ \defeq \sum_{\srcStr \in \srcStrings} \pSource(\srcStr) \potential(\srcStr)\,.
\end{equation}
For a transduced model, we take the potential to be the predicate \isMember{} (\cref{sec:tlms}), $\potential(\srcStr)=\isMember(\srcStr,\tgtStr)$, for a fixed $\tgtStr$.
The estimand $\evZ$ is then the probability that a string drawn from the source model maps to a target string having $\tgtStr$ as a prefix, i.e., $\evZ=\prefixTarget(\tgtStr)$.

The algorithm estimating $\evZ$ with SMC is shown in \cref{fig:smc}, and we now describe it.
The pool $\queue$ holds $M$ \defn{particles}, each consisting of a source prefix and a particle weight, initialized at the empty string with weight $1/M$ (\cref{line:smcinit}).
At every iteration, a particle draws its next symbol by \defn{sequential importance sampling} \citep{smcbook}, drawing from a \defn{conditional proposal} distribution $\prefixProposal(\cdot\mid\srcStr)$ (\cref{line:smcdraw}) and correcting for it by multiplying its weight by the ratio of source to proposal conditionals (\cref{line:smcweight}),
\begin{equation}\label{eq:sis-weight-update}
  \weight \gets \weight \cdot \frac{\prefixSource(\srcSymp \mid \srcStr)}{\prefixProposal(\srcSymp \mid \srcStr)}.
\end{equation}
A particle that draws $\eos$ is \defn{completed}: it leaves the pool, and its weight, multiplied by the potential, is added to the estimate $\estZ$ (\cref{line:smceos,line:smcreturn}).
Assume for a moment that we do not call \algResample.
Then, for a single particle, the factors of \cref{eq:sis-weight-update}, the potential, and the initial $1/M$ multiply to $\tfrac{1}{M}\weight(\srcStr)$, with the \defn{importance weight}
\begin{equation}
  \label{eq:is-weight}
  \weight(\srcStr)\defeq\potential(\srcStr)
  \frac{\prefixSource(\eos\mid\srcStr)}{\prefixProposal(\eos\mid\srcStr)}
  \prod_{s=1}^{|\srcStr|}
  \frac{\prefixSource(\srcStr_s\mid\srcStr_{<s})}
  {\prefixProposal(\srcStr_s\mid\srcStr_{<s})}
  =\frac{\potential(\srcStr)\pSource(\srcStr)}{\proposal(\srcStr)}.
\end{equation}
Here $\proposal(\srcStr)\defeq\prefixProposal(\eos\mid\srcStr)\prod_{s=1}^{|\srcStr|}\prefixProposal(\srcStr_s\mid\srcStr_{<s})$.
We further require that
\begin{equation}\label{eq:support}
  \potential(\srcStr)\,\pSource(\srcStr)>0 \implies \proposal(\srcStr)>0
  \qquad\text{for every }\srcStr\in\srcStrings,
\end{equation}
to ensure that every string with positive potential--weighted mass can be sampled.

Without the \algResample{} call the particles are independent and the estimate is the average $\estZ=\frac1M\sum_{m=1}^{M}\weight(\srcStr^{(m)})$ of $M$ draws of \cref{eq:is-weight}.

Without resampling, the particle weights can become increasingly unequal as source prefixes grow, until a single particle dominates \citep{smcbook}.
The \defn{effective sample size}
\begin{equation}
  \label{eq:ess}
  \text{ESS} = \frac{\left(\sum_{m=1}^{|\queue|} \weight^{(m)}\right)^2}{\sum_{m=1}^{|\queue|}\left(\weight^{(m)}\right)^2}
\end{equation}
quantifies that imbalance.
When it falls below $\threshold |\queue|$ (\cref{line:smcesstest}), for a fixed $0<\threshold<1$, \algResample{} redraws the pool with replacement proportionally to weight and resets every weight to the pool mean (\cref{line:smcredraw,line:smcreset}), so the higher-weighted particles are duplicated and the lower ones discarded.
When the support condition \cref{eq:support} holds and the run halts almost surely, the estimator is unbiased without resampling. Since multinomial resampling with replacement preserves weighted sums in conditional expectation, \algSMC{} remains unbiased.

\subsection{Intermediate Targets and the \algSMCrb{} Algorithm}
\label{sec:intermediate-targets}
In \cref{sec:resampling}, particles are extended until $\eos$ is sampled and only then is the potential evaluated.
\algBeam{} instead tests prefixes as they are extended and drops a prefix once it can no longer cover the target.
\algSMCrb{} applies SMC within \algBeam{}'s target recursion.
Like \algBeam{}, \algSMCrb{} applies the \isCylinder{} and \isMember{} checks to add quotient and remainder contributions.
It samples one live next source symbol per particle instead of enumerating all of them.
The recursion iterates over target prefixes while each call extends source prefixes.
To express these updates as importance sampling, associate each target position $t$ with the \defn{intermediate target}
\begin{equation}
  \interTarget_{t}(\srcStr)\defeq \prefixSource(\srcStr)\cdot \twist_{t}(\srcStr).
\end{equation}

As in \cref{sec:two-algorithms}, we rely on the live twist $\twist_t(\srcStr,\tgtStr)\defeq\isLive(\srcStr,\tgtStr_{\le t})$, written $\twist_t(\srcStr)$ when $\tgtStr$ is clear.
Sequential importance sampling requires each intermediate target to place mass only where the preceding target does \citep{delmoral2006smc,naesseth2019elements}.
This holds because a prefix that can still cover $\tgtStr_{\le t}$ can also cover $\tgtStr_{<t}$.
\begin{figure}[t]
  \centering
  \begin{tikzpicture}[>=Stealth,
      srcbox/.style={draw=ETHGray!70, fill=ETHGray!8, minimum width=7.5mm, minimum height=5.5mm, inner sep=1pt},
      tgtbox/.style={draw=slatelake!80, fill=slatelake!12, minimum width=7.5mm, minimum height=5.5mm, inner sep=1pt},
      emit/.style={->, slatelake, line width=0.8pt},
      slbl/.style={font=\footnotesize\color{ETHGray}},
      live/.style={circle, draw=jademist!90, inner sep=0pt, minimum size=3.4mm},
      dead/.style={ETHGray!70, dashed, line width=0.5pt}]
    \node[slbl, anchor=west] at (-0.7,2.9) {(a)};
    \node[srcbox] (s1) at (0,2.2)   {$\srcSym_1$};
    \node[srcbox] (s2) at (1.4,2.2) {$\srcSym_2$};
    \node[srcbox] (s3) at (2.8,2.2) {$\srcSym_3$};
    \node[srcbox] (s4) at (4.2,2.2) {$\srcSym_4$};
    \node[srcbox, dashed] (se) at (5.8,2.2) {$\eos$};
    \draw[->, ETHGray!80, dashed] (s4.east) -- (se.west);
    \node[slbl] (e1) at (0,0.6)   {$\varepsilon$};
    \node[tgtbox] (t1) at (1.4,0.6) {$\tgtSym_1$};
    \node[slbl] (e2) at (2.8,0.6) {$\varepsilon$};
    \node[tgtbox] (t2) at (4.2,0.6) {$\tgtSym_2$};
    \node[tgtbox] (t3) at (5.6,0.6) {$\tgtSym_3$};
    \draw[->, slatelake!45, line width=0.8pt] (s1.south) -- (e1.north);
    \draw[emit] (s2.south) -- (t1.north);
    \draw[->, slatelake!45, line width=0.8pt] (s3.south) -- (e2.north);
    \draw[emit] (s4.south) -- (t2.north);
    \draw[emit] (s4.south) -- (t3.north);
    \draw[->, ETHGray!80] (-0.55,1.75) -- (-0.55,1.05);
    \node[slbl, anchor=east] at (-0.62,1.4) {$\transFn$};
    \foreach \x in {2.1, 4.9, 6.3} {
        \draw[ETHGray!60, dashed] (\x,0.15) -- (\x,1.15);
      }
    \begin{scope}[xshift=8.8cm, yshift=1.6cm]
      \node[slbl, anchor=west] at (-0.7,1.5) {(b)};
      \node[live, fill=jademist!60] (r)  at (0,0)      {};
      \node[live, fill=jademist!45] (a1) at (1.3,0.75)  {};
      \node[live, fill=jademist!18] (a2) at (2.6,1.05)  {};
      \node[live, fill=jademist!6]  (a3) at (3.9,1.2)   {};
      \draw[jademist, line width=2.4pt] (r) -- (a1);
      \draw[jademist, line width=2.2pt] (a1) -- (a2);
      \draw[jademist, line width=2.0pt] (a2) -- (a3);
      \node[live, fill=jademist!75] (b1) at (1.3,-0.75) {};
      \node[live, fill=jademist!85] (b2) at (2.6,-1.0)  {};
      \node[live, fill=jademist!95] (b3) at (3.9,-1.15) {};
      \draw[jademist, line width=0.8pt] (r) -- (b1);
      \draw[jademist, line width=0.8pt] (b1) -- (b2);
      \draw[jademist, line width=0.8pt] (b2) -- (b3);
      \draw[dead] (r) -- (1.2,0);   \node[ETHGray, font=\footnotesize] at (1.32,0) {$\times$};
      \draw[dead] (a1) -- (2.5,0.45); \node[ETHGray, font=\footnotesize] at (2.62,0.45) {$\times$};
      \draw[dead] (b1) -- (2.5,-0.4); \node[ETHGray, font=\footnotesize] at (2.62,-0.4) {$\times$};
    \end{scope}
    \draw[jademist, line width=2.4pt] (0,-0.75) -- (0.55,-0.75);
    \node[slbl, anchor=west] at (0.65,-0.75) {source mass $\prefixSource(\srcSymp\mid\srcStr)$};
    \node[live, fill=jademist!12] at (4.6,-0.75) {};
    \node[live, fill=jademist!85] at (5.05,-0.75) {};
    \node[slbl, anchor=west] at (5.3,-0.75) {mass ahead $\twistOptimal_t$ (\cref{eq:optimal-twist})};
    \draw[dead] (9.4,-0.75) -- (9.9,-0.75);
    \node[ETHGray, font=\footnotesize] at (10.05,-0.75) {$\times$};
    \node[slbl, anchor=west] at (10.25,-0.75) {not live};
  \end{tikzpicture}
  \caption{(a)~An example source string aligned to a target prefix under a function $\transFn$, the dashed lines marking the boundaries between target positions.
    A symbol may emit none or several target symbols, so the source spans between boundaries vary in length while the alignment stays monotone.
    At each step a \emph{source} symbol is scanned, and at each boundary a \emph{target} symbol is emitted.
    (b)~When $\twist=\isLive$ we sample by source mass (edge width) and never enter a dead branch ($\times$), so no draw is wasted. Since we only look at the current live symbols, we cannot tell a heavy branch with little mass ahead (top) from a light branch that carries it (bottom).
    Importance weighting corrects for the proposal, while resampling with replacement mitigates particle-weight degeneracy (\cref{sec:resampling,sec:smc-single-rb}).}
  \label{fig:alignment}
\end{figure}
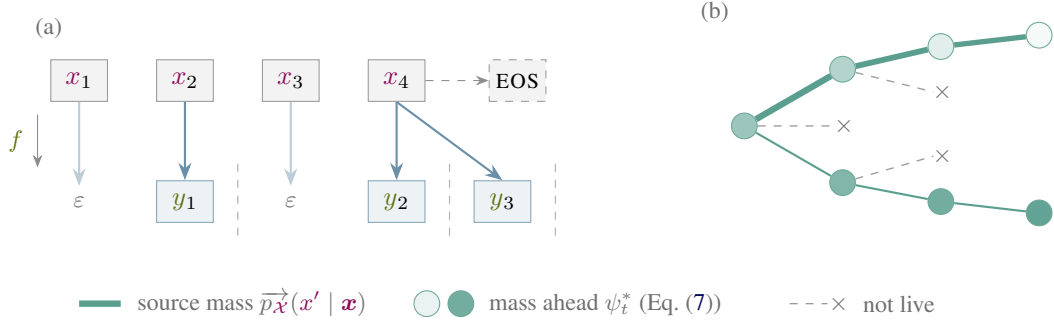

Recall the live set $\liveset{\srcStr}$ of next source symbols (\cref{sec:two-algorithms}), taken here relative to $\tgtStr_{\le t}$, and define the \defn{twisted live mass}
\begin{equation}
  \liveZ_t(\srcStr)
  \defeq
  \sum_{\srcSymp'\in\liveset{\srcStr}}
  \prefixSource(\srcSymp'\mid\srcStr)\,\twist_t(\srcStr\srcSymp').
\end{equation}
The proposal distribution draws the next symbol with probability proportional to $\prefixSource(\srcSymp\mid\srcStr)\,\twist_{t}(\srcStr\srcSymp)$,
\begin{equation}
  \label{eq:proposal-tlm}
  \begin{split}
    &\prefixProposalt(\srcSymp\mid\srcStr)
    =\frac{\prefixSource(\srcSymp\mid\srcStr)\,\twist_{t}(\srcStr\srcSymp)}{\propZ_t(\srcStr)},
    \qquad
    \prefixProposalt(\eos\mid\srcStr)
    =\frac{\indicator{\isMember(\srcStr,\tgtStr_{\le t})}\,\prefixSource(\eos\mid\srcStr)}{\propZ_t(\srcStr)},
    \\
    &\propZ_t(\srcStr)\defeq\indicator{\isMember(\srcStr,\tgtStr_{\le t})}\,\prefixSource(\eos\mid\srcStr)+\!\!\sum_{\srcSymp'\in\liveset{\srcStr}}\!\!\prefixSource(\srcSymp'\mid\srcStr)\,\twist_t(\srcStr\srcSymp'),
  \end{split}
\end{equation}
For the live twist, $\liveZ_t(\srcStr)$ reduces to the raw conditional probability of the live set, so every particle stays live for $\tgtStr_{\le t}$ by construction.
\algSMC{} uses the same live-set restriction.
The downside is that the live set carries no information about how much covering mass lies ahead, so a particle may look good early and carry little weight later (\cref{fig:alignment}b).
We accept the compromise because the live predicate is computable at every step (\cref{sec:background}), whereas the full mass ahead is the optimal twist of \cref{eq:optimal-twist}, exactly the quantity we cannot compute.

\paragraph{Importance sampling in the target recursion.}
\algSMCrb{} iterates over the target and source indices as shown in the example \cref{fig:alignment}a.
A \defn{target step} occurs once when the recursion advances from target position $t-1$ to $t$.
At a fixed source prefix, the weight becomes
\begin{equation}\label{eq:target-step}
  \weight \;\gets\; \weight\cdot\frac{\interTarget_{t}(\srcStr)}{\interTarget_{t-1}(\srcStr)}
  =\weight\cdot\frac{\prefixSource(\srcStr)\twist_t(\srcStr)}{\prefixSource(\srcStr)\twist_{t-1}(\srcStr)}
  =\weight\cdot\frac{\twist_t(\srcStr)}{\twist_{t-1}(\srcStr)}.
\end{equation}
The \defn{source step} advances by drawing one symbol from the twisted proposal.
Its incremental importance weight is
\begin{equation}
  \begin{aligned}
    \frac{\interTarget_t(\srcStr\srcSymp)/\interTarget_t(\srcStr)}
    {\prefixProposalt(\srcSymp\mid\srcStr)}
     & =
    \frac{\prefixSource(\srcSymp\mid\srcStr)\twist_t(\srcStr\srcSymp)}
    {\twist_t(\srcStr)} \cdot
    \frac{\propZ_t(\srcStr)}
    {\prefixSource(\srcSymp\mid\srcStr)\twist_t(\srcStr\srcSymp)} =\frac{\propZ_t(\srcStr)}{\twist_t(\srcStr)}.
  \end{aligned}
\end{equation}
Therefore,
\begin{equation}
  \srcSymp\sim\prefixProposalt(\cdot\mid\srcStr),\qquad
  (\srcStr,\weight)
  \gets
  \left(\srcStr\srcSymp,
  \weight\frac{\propZ_t(\srcStr)}{\twist_t(\srcStr)}\right),
\end{equation}
with $\propZ_t(\srcStr)$ the proposal's normalizer of \cref{eq:proposal-tlm}.
At position $t$, the target-step factor is combined with one source-step factor for every symbol sampled within the call.
A particle that ends position $t-1$ at the prefix $\srcStr=\srcStr_{\le n}$, $n=|\srcStr|$, and takes $k$ source steps through position $t$ extends it to $\srcStr_{\le n+k}$, drawing the symbol $\srcStr_{n+j}$ at step $j$, and updates its weight by
\begin{equation}
  \weight \;\gets\; \weight\cdot
  \frac{\interTarget_{t}(\srcStr_{\le n+k})}{\interTarget_{t-1}(\srcStr_{\le n})\,\prod_{j=1}^{k}\prefixProposalt(\srcStr_{n+j}\mid\srcStr_{<n+j})}
  =\weight\cdot\frac{\twist_t(\srcStr_{\le n})}{\twist_{t-1}(\srcStr_{\le n})}\,
  \prod_{j=1}^{k}\frac{\propZ_t(\srcStr_{<n+j})}{\twist_t(\srcStr_{<n+j})},
\end{equation}
the target step once and one source-step factor per symbol.

For the live twist, $\twist_t(\srcStr)=1$ on every surviving particle.
\algSMCrb{} marginalizes the $\eos$ decision and conditions the proposal of \cref{eq:proposal-tlm} on drawing a live next source symbol.
Substituting the conditional proposal into \cref{eq:sis-weight-update}, the update simplifies to
\begin{equation}
  \label{eq:weight-update-tlm}
  \weight^{(m)}\gets
  \weight^{(m)}\cdot \liveZ_t(\srcStr^{(m)}).
\end{equation}
The sampled particle's weight therefore equals the sum of the weights of its parent's live children.

For a quotient particle carried from position $t-1$, the target-step factor $\twist_t(\srcStr)/\twist_{t-1}(\srcStr)$ is one if the prefix remains live at the new position and zero otherwise.
It re-enters the pool unchanged when the factor is one. When the factor is zero, $\liveZ_t(\srcStr)=0$, so \algSMCrb{} drops it at its next source step.
The recursive call also returns completed remainder entries. The membership re-check retains one exactly when its source string covers the longer target.
Together, these operations implement the update from $\interTarget_{t-1}$ to $\interTarget_t$ (\cref{fig:algos-smc}).\looseness=-1

\begin{figure}[t]
  \centering
  \smaller
  \begin{minipage}[t]{0.48\linewidth}
    \begin{minted}[autogobble,xleftmargin=14pt]{python}
def @$\algSMCrb(\tgtStr, M, \twist, \pruneFunc)$@:
  @$N \gets |\tgtStr|$@
  if @$N = 0$@:
    @$(\algQuotientp, \algRemainderp), \estZ \gets ([(\varepsilon, \tfrac{1}{M})]_{m=1}^{M},\ \emptyset),\ [\,]$@
  else:
    @$(\algQuotientp, \algRemainderp), \estZ  \gets \algSMCrb(\tgtStr_{<N}, M, \twist, \pruneFunc)$@
  @$\queue \gets \textsc{Queue}(\algQuotientp)$@
  @$(\algQuotient, \algRemainder), \accum \gets (\emptyset, \emptyset),\ 0$@
  for @$(\srcStr, \weight) \in \algRemainderp$@:
    if @$\isMember(\srcStr, \tgtStr)$@:
      @$\algRemainder$@.add(@$(\srcStr, \weight)$@)
      @$\accum \texttt{ += } \weight$@
  while @$|\queue| > 0$@:
    @$\queuep \gets \emptyset$@
    for @$\srcStr,\weight \in \queue$@:
      if @$\isCylinder(\srcStr, \tgtStr)$@:
        @$\algQuotient$@.add(@$(\srcStr, \weight)$@)
        @$\accum \texttt{ += } \weight$@
        continue
      if @$\isMember(\srcStr, \tgtStr)$@:
        @$\algRemainder$@.add(@$(\srcStr, \weight \cdot \prefixSource(\eos \mid \srcStr))$@)
        @$\accum \texttt{ += } \weight\cdot \prefixSource(\eos \mid \srcStr)$@
      @$\liveZ \gets \sum_{\srcSymp \in \srcAlphabet} \prefixSource(\srcSymp \mid \srcStr) \cdot \twist(\srcStr\srcSymp)$@
      if @$\liveZ = 0$@:
        continue
      @$\textbf{\texttt{sample }}\srcSymp \sim \prefixSource(\cdot \mid \srcStr)\,\twist(\srcStr\cdot) / \liveZ$@
      @$\queuep$@.add(@$(\srcStr \srcSymp,\ \frac{\weight \cdot \liveZ}{\twist(\srcStr\srcSymp)})$@)
    @$\queue \gets \pruneFunc(\queuep, \accum)$@
  @$\estZ$@.append(@$\accum$@)
  return @$(\algQuotient, \algRemainder), \estZ$@
\end{minted}
  \end{minipage}
  \hfill
  \begin{minipage}[t]{0.48\linewidth}
    \begin{minted}[autogobble,xleftmargin=10pt]{python}
def @$\algResample_{\twist,\threshold}(\queuep)$@:
  @$\mathcal{W}^{(i)} \gets \weight^{(i)} \cdot \twist(\srcStr^{(i)})$@
  @$V \gets \sum_i \mathcal{W}^{(i)}$@
  @$N \gets |\queuep|$@
  if @$V^2 \ge \threshold\, N \sum_i (\mathcal{W}^{(i)})^2$@:
    return @$\queuep$@
  @$j_1, \ldots, j_N \simIID \mathrm{Categorical}(\mathcal{W}^{(i)} / V)$@
  return @$\{(\srcStr^{(j_n)}, \frac{V\,\weight^{(j_n)}}{N\,\mathcal{W}^{(j_n)}})\}_{n=1}^{N}$@
\end{minted}
  \end{minipage}
  \caption{%
    The \algSMCrb{} algorithm and ESS-gated multinomial resampling under a twist.
  }
  \label{fig:algos-smc}
\end{figure}

\subsection{Unbiasedness of \algSMCrb{}}
\label{sec:smc-single-rb}
We now show that \algSMCrb{} is unbiased. The proof follows \cref{sec:swor-correctness} for \algBeam{} but with the sampled symbol extension instead of all live extensions.

\begin{proposition}[Unbiasedness of \algSMCrb{}]
  \label{prop:smcrb-unbiased}
  Fix a nonempty target string $\tgtStr$ and run \algSMCrb{} (\cref{fig:algos-smc}) with a twist.
  Suppose that every prune preserves the pool's expected weight at every prefix (\cref{eq:prune-preserve}) and that the run halts almost surely.
  Then $\mathbb{E}[\estZ_{|\tgtStr|}]=\prefixTarget(\tgtStr)$.
\end{proposition}
\begin{proof}
  We follow the nested induction used in \cref{lem:pruned-proper}, with an outer induction on target length and inner induction on the within call iterations.
  Write $\queue_i$ for the \algSMCrb{} pool at iteration $i$ and $\queueBar_i$ for the corresponding pool in the prune-free \algBeam{} run.
  In the outer base case, the $M$ copies of $\varepsilon$ with weight $1/M$ in \algSMCrb{} have the same aggregate weight as the single copy of weight one in \algBeam{}.
  The recursive initial-pool argument is otherwise unchanged, so it remains to replace the extension identity \cref{eq:extension}.

  Condition on the current pool and fix a non-cylinder particle at $\srcStr$ with weight $\weight$.
  Write $\liveZ(\srcStr)=\sum_{\srcSympp}\prefixSource(\srcSympp\mid\srcStr)\twist(\srcStr\srcSympp)$ for its twisted live mass.
  When $\liveZ(\srcStr)>0$, let $\rvSym$ be drawn from the proposal of \cref{eq:proposal-tlm}, conditioned on a live source symbol.
  For any next source symbol $\srcSymp$ with $\twist(\srcStr\srcSymp)>0$, its conditional expected contribution is
  \begin{equation}
    \label{eq:sampled-cancel}
    \mathbb{E}_{\rvSym}\!\left[
      \idfunc\{\rvSym=\srcSymp\}
      \frac{\weight\,\liveZ(\srcStr)}{\twist(\srcStr\srcSymp)}
      \right]
    =
    \frac{\prefixSource(\srcSymp\mid\srcStr)\twist(\srcStr\srcSymp)}{\liveZ(\srcStr)}
    \frac{\weight\,\liveZ(\srcStr)}{\twist(\srcStr\srcSymp)}
    =\weight\,\prefixSource(\srcSymp\mid\srcStr).
  \end{equation}
  A zero-twist child receives no contribution under either extension.
  If $\liveZ(\srcStr)=0$, both extensions assign every admitted child zero weight, and cylinder prefixes produce no children under either extension.
  Summing \cref{eq:sampled-cancel} over the particles at $\srcStr$ gives the conditional-expectation version of \cref{eq:extension}.
  The tower rule and \cref{eq:prune-preserve} then close the same within-call induction and give, for every iteration $i$ and source prefix $\srcStr$,
  \begin{equation}
    \expected{}{\weight_{\queue_i}(\srcStr)}
    =\weight_{\queueBar_i}(\srcStr).
  \end{equation}

  Marginalizing the stop decision replaces its sampled contribution by its conditional expectation.
  At a member, the baseline proposal stops with probability $\prefixSource(\eos\mid\srcStr)/\propZ_t(\srcStr)$ and returns weight $\weight\propZ_t(\srcStr)$, so\looseness=-1
  \begin{equation}
    \label{eq:stop-cancel}
    \frac{\prefixSource(\eos\mid\srcStr)}{\propZ_t(\srcStr)}
    \weight\,\propZ_t(\srcStr)
    =\weight\,\prefixSource(\eos\mid\srcStr),
  \end{equation}
  which is the remainder contribution added by \algSMCrb{}.
  \algSMCrb{} adds quotient contributions and re-queues prefixes exactly as \algBeam{} does.
  The same assembly argument as in \cref{thm:swor-unbiased} therefore applies to the pool identity above, and \cref{eq:assembly-swap,eq:assembly-means} give the result.\looseness=-1
\end{proof}
The ESS-gated multinomial resample of \cref{fig:algos-smc} satisfies \cref{eq:prune-preserve} for every fixed input pool \citep[Prop.~16.3]{chopin2020introduction}.

\subsection{A Common Tail-Roulette Prune}
\label{sec:halting}
Source-prefix enumeration can continue indefinitely even as the remaining particle weights tend to zero (\cref{sec:sworb}).
The decorator in \cref{fig:tail-roulette} independently removes each sufficiently light particle with probability $1/2$ and doubles its weight otherwise.
For use with this decorator, we represent each particle as a triple $(\srcStr,\weight,\weight_0)$.
At the start of a target-prefix call, we set $\weight_0=\weight$ for every particle in $\queue_1$.
Each child and resampled copy inherits $\weight_0$.
The original pruning function acts on the source prefix and current weight while keeping $\weight_0$ unchanged.
For $\kappa\in(0,1)$, the decorator first applies the \emph{original} pruning function, then applies the coin flip independently to each returned particle $(\srcStr,\weight,\weight_0)$ with $\weight<\kappa\weight_0$.\looseness=-1

\begin{figure}[t]
  \centering
  \begin{minipage}{\linewidth}
    \begin{minted}[autogobble,xleftmargin=14pt]{python}
def with_tail_roulette(@$\pruneFunc,\kappa$@):
  def decorated(@$\queuep,\ldots$@):
    @$\queue \gets \pruneFunc(\queuep,\ldots)$@
    @$\widetilde{\queue} \gets \emptyset$@
    for @$(\srcStr^{(m)},\weight^{(m)},\weight_0^{(m)})\in\queue$@:
      if @$\weight^{(m)} = 0$@:
        continue
      if @$\weight^{(m)} \ge \kappa\,\weight_0^{(m)}$@:
        @$\widetilde{\queue}$@.add(@$(\srcStr^{(m)},\weight^{(m)},\weight_0^{(m)})$@)
        continue
      @$b\sim\mathrm{Bernoulli}(1/2)$@
      if @$b=1$@:
        @$\widetilde{\queue}$@.add(@$(\srcStr^{(m)},2\weight^{(m)},\weight_0^{(m)})$@)
    return @$\widetilde{\queue}$@
  return decorated
\end{minted}
  \end{minipage}
  \caption{Tail-roulette decorator for a pruning function.}
  \label{fig:tail-roulette}
\end{figure}

Fix a finite starting pool $\queue_1$, $\kappa\in(0,1)$, and an original pruning function that returns at most $M$ particles.
We first show that the decorator preserves the particle weight in expectation whenever the original pruning function does.
We then show that the algorithms with a fixed particle cap $M$ in \cref{tab:tlm-zoo} halt almost surely.\looseness=-1

\paragraph{Expected-weight preservation.}
Conditional on the original pruning function's output and the reference weights $\weight_0$, a particle selected for the coin flip contributes $2\weight$ with probability $1/2$ and zero otherwise, giving expected contribution $\weight$.
The coin flip therefore preserves every prefix weight in conditional expectation.
Averaging over the original pruning function's output, the decorated pruning function satisfies \cref{eq:prune-preserve} whenever the original does.

\paragraph{Almost-sure halting.}
Let $\queue_i^\circ$ be the pool returned by the original pruning function at iteration $i$, immediately before the coin flips, and define its total particle weight by
\begin{equation}
  W_i
  \defeq
  \sum_{(\srcStr^{(m)},\weight^{(m)},\weight_0^{(m)})\in\queue_i^\circ}
  \weight^{(m)}.
\end{equation}
By \cref{eq:prune-preserve,lem:pruned-proper}, $\expected{}{W_i}$ equals the total weight in the corresponding prune-free pool after $i$ source extensions for \algBeam{} and hence \algSWOR{}.
The sampled-extension identity \cref{eq:sampled-cancel} gives the same result for \algSMCrb{}, and the importance-weight identity \cref{eq:sis-weight-update,eq:is-weight} gives it for \algSMC{}.
For $\algTopM$, the expected total is no larger because its prune only removes particles.
In every case,
\begin{align}
  \expected{}{W_i}
  &\le
  \sum_{m=1}^{|\queue_1|}
  \weight_0^{(m)}
  \sum_{\srcStrp\in\srcAlphabet^i}
  \prefixSource(\srcStrp\mid\srcStr^{(m)})
  \xrightarrow{i\to\infty} 0.
  \label{eq:tail-roulette-mass}
\end{align}
The inner sum is the conditional probability of generating at least $i$ additional source symbols from $\srcStr^{(m)}$.
Because $\pSource$ is a distribution over finite strings, this probability tends to zero as $i\to\infty$ (\cref{sec:lms}).
Since the outer sum has finitely many terms, the displayed limit follows.

Discard zero-weight particles from $\queue_1$.
If none remain, the call halts immediately.
Otherwise, set
\begin{equation}
  \delta_{\min}
  \defeq
  \kappa\min_{1\le m\le|\queue_1|}\weight_0^{(m)}
  >0.
\end{equation}
Because descendants inherit $\weight_0$, every descendant's coin-flip threshold $\kappa\weight_0$ is at least $\delta_{\min}$.

For $i\ge1$, write $\widetilde{\queue}_i$ for the pool after the coin flips at iteration $i$.
If $W_i<\delta_{\min}$, every particle in $\queue_i^\circ$ is selected for a coin flip.
The original pruning function returns at most $M$ particles, so the independent coin flips empty the pool with conditional probability at least $2^{-M}$.
Since $W_i$ is nonnegative, Markov's inequality and the coin-flip bound above give
\begin{align}
  \Pr{W_i\ge\delta_{\min}}
    &\le
    \frac{\expected{}{W_i}}{\delta_{\min}}
    \longrightarrow 0, \\
  \Pr{\widetilde{\queue}_i\ne\emptyset}
    &\le
    \Pr{W_i\ge\delta_{\min}}
    +(1-2^{-M})
    \Pr{\widetilde{\queue}_{i-1}\ne\emptyset},
    \qquad i\ge2.
\end{align}
Let $E_i\defeq\{\widetilde{\queue}_i\ne\emptyset\}$ be the event that the pool remains nonempty after iteration $i$, and let $p_i\defeq\Pr{E_i}$.
Because an empty pool remains empty, $E_{i+1}\subseteq E_i$.
Thus, $(p_i)_{i\ge1}$ is nonincreasing and bounded below by zero, so it converges to some $p\ge0$.
Taking $i\to\infty$ in the recursion, with $p_i\to p$, $p_{i-1}\to p$, and $\Pr{W_i\ge\delta_{\min}}\to0$, gives $p\le(1-2^{-M})p$.
Since $1-2^{-M}<1$, this forces $p=0$.
The intersection $\bigcap_{i\ge1}E_i$ is the event that the pool remains nonempty at every iteration.
For decreasing events, its probability is the limit of their probabilities, so $\Pr{\bigcap_{i\ge1}E_i}=\lim_{i\to\infty}p_i=0$.
A call using the decorated prune thus halts almost surely.
For the recursive algorithms, the target string has finitely many positions, so the complete run also halts almost surely.
Any $\kappa\in(0,1)$ gives the halting guarantee.\looseness=-1

\paragraph{Settings.}
The threshold $\kappa$ determines when roulette begins.
We choose it small so that roulette is delayed until a particle's current weight is small relative to its reference weight.
We use $\kappa=e^{-15}$ for the WikiText experiments and $\kappa=e^{-30}$ for the DNA and MECO experiments.
\sworB{} does not require the decorator because its adaptive prune uses $\accum_{\epsilon,t}$ to ensure almost-sure halting (\cref{prop:sworb}).\looseness=-1


\section{Sensitivity to Pruning Parameters}
\label{app:knobs}

The experiments use $\epsilon=0$.
For each fixed target prefix, the finite-context source models and finite-state transducers give finitely many live configurations.
Full support gives a positive lower bound on the probability of reaching a member or cylinder from any such configuration, so $\accum$ becomes positive almost surely.
The proof of \cref{prop:sworb} then applies with $\accum_{\epsilon,t}=\accum>0$.
Besides $M$, the adaptive-budget prune uses $\rho$ (\cref{sec:sworb}), and the ESS-gated multinomial prune uses $\threshold$ (\cref{app:smc}).
We vary both on the ten-paragraph bigram corpus of \cref{sec:ptb-lane}, using eight seeds (\cref{fig:knob-pareto}).

Smaller $\rho$ retains more particles.
At nominal $M{=}8000$, reducing $\rho$ from $1$ to $0.02$ lowers the standard deviation from $0.711$ to $0.073$ nats and raises the time per seed from $349$ to $1551$ seconds.
The default $\rho{=}0.1$ gives a standard deviation of $0.200$ nats at $623$ seconds.
Plain \algSWOR{} at $M{=}2000$ gives $0.170$ nats at $3432$ seconds.
With $149$ seeds, the standard deviation at the default is $0.163$ with CI $[0.146,0.184]$, consistent with the fixed-$M$ \algSWOR{} result at $5.5\times$ less CPU time.
At $M{=}2000$, the between-seed standard deviation rises from $0.632$ nats at $\rho{=}0.1$ to $4.815$ at $\rho{=}0.5$ and $5.283$ at $\rho{=}1$.
At $\rho{=}1$, the mean live count is $18$, and most log estimates are several nats low.

Increasing $\threshold$ causes more frequent resampling, while CPU time changes little across the tested values.
At $M{=}2000$, the survivor-conditioned standard deviation rises from $5.1$ nats at $\threshold{=}0.25$, where seven of eight corpus seeds survive, to $7.86$ at the default $\threshold{=}0.5$ and $9.8$ at $\threshold{=}0.75$.
At $\threshold{=}0$, every seed has at least one zero paragraph estimate.
The same holds at $M{=}200$ for every tested value of $\threshold$, despite some nonzero individual paragraph estimates.

\begin{figure}[t]
  \centering
  \includegraphics[width=\linewidth]{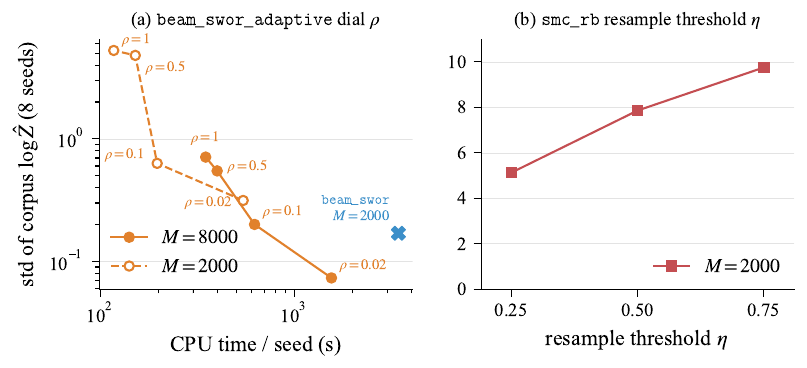}
  \caption{Sensitivity to the pruning parameters on the ten-paragraph bigram corpus (eight seeds).
    (a)~Between-seed standard deviation of the corpus $\log\estZ$ against CPU time per seed as $\rho$ varies for \sworB{} at nominal $M{\in}\{2000,8000\}$. Labels give $\rho$, and plain \algSWOR{} at $M{=}2000$ is the fixed-$M$ reference.
    (b)~Survivor-conditioned standard deviation as the ESS threshold $\threshold$ varies for \algSMCrb{} at $M{=}2000$. A corpus seed survives when all ten paragraph estimates are nonzero. Annotations mark settings with fewer surviving seeds and note that none survives at $M{=}200$.}
  \label{fig:knob-pareto}
\end{figure}

\section{Exact Ground Truth for Bigram Models}
\label{app:tailacc-solve}

The bigram model can be represented as a weighted finite-state acceptor.
Composing it with $\STPTB$ and projecting onto the target labels gives a weighted automaton defining $\pbigram{\circ}\STPTB$.
In the representation used here, the functional transducer $\STPTB$ admits at most one accepting path for each source string, so path sums do not double-count source strings.

For each state $\state$, let $\beta(\state)$ be its backward weight, the sum of the weights of all paths starting at $\state$, including the final weight at the end of each path.
Let $T$ be the transition-weight matrix and $f$ the final-weight vector.
From each state, a path either stops there with its final weight or takes a transition and then continues.
The backward weights therefore satisfy \citep{Mohri2009}
\begin{equation}
  \beta
  =
  f+T\beta,
  \qquad\text{so}\qquad
  \beta
  =
  (I+T+T^2+\cdots)f
  =
  (I-T)^{-1}f.
\end{equation}

Because the composed automaton has no cycle that can be traversed without consuming a source symbol, the number of consumed source symbols tends to infinity with the number of transitions.
The probability of generating at least $k$ source symbols tends to zero as $k\to\infty$, so the series converges.

For each position $t$ and state $\state$, let $\alpha_t(\state)$ be the sum of the weights of all paths that start at the initial state, emit $\tgtStr_{\le t}$, and end in $\state$ immediately after emitting $\tgtStr_t$.
Summing over the states reached after emitting $\tgtStr_t$ gives
\begin{equation}
  \prefixTarget(\tgtStr_{\le t})
  =
  \sum_{\srcStr\in\preCover(\tgtStr_{\le t})}\pSource(\srcStr)
  =
  \sum_{\state\in\States}\alpha_t(\state)\beta(\state)
  =
  \alpha_t^\top\beta.
\end{equation}

The backward weights $\beta(\state)$ are computed once for all states, while $\alpha_t(\state)$ is updated after each target symbol.
The composed state space is finite for an $n$-gram source, but not for a neural language model whose next-symbol probabilities depend on the complete source prefix.

\section{The Genetic-Code Worst Case}
\label{app:worstcase}
This appendix analyzes a worst case for the DNA setting in \cref{sec:dna}. Standard beam search captures an exponentially vanishing share of the target mass, while \algSMC{}, \algSMCrb{}, and fixed-$M$ \algSWOR{} are exact for any $M$.
Consider a uniform source model $\pSource(\srcStr) = (1/4)^{3T}$ over source strings of length $3T$ on a four-symbol alphabet, and a target of length $T$ where each target symbol is the image of $d$ length-three source blocks (``codons'').
Assume that the live proposal samples each of the $d$ codons with probability $1/d$. This holds, for example, for repeated isoleucine, where $d=3$ and the three codons differ only in their final base.
Exactly $d^{T}$ strings in the source model's support belong to the precover, so $\pTarget(\tgtStr) = d^{T} \cdot (1/4)^{3T}$.

\paragraph{Standard beam search.}
\algBeam{} with $\algTopM$ retains at most $M$ of these $d^{T}$ source strings, yielding
\begin{equation}
  \estZ_{\text{beam}} \le M \cdot \left(\frac{1}{4}\right)^{3T},
  \qquad
  \frac{\estZ_{\text{beam}}}{\pTarget(\tgtStr)} \le \frac{M}{d^{T}}.
\end{equation}
This ratio vanishes exponentially in $T$.
For $T = 100$ and $d = 3$, the smallest beam that recovers all the mass has $M = 3^{100} \approx 5 \times 10^{47}$.
The estimate is biased for any $M<d^{T}$, systematically underestimating $\pTarget(\tgtStr)$.

\paragraph{Importance-weighted estimators.}
Each valid codon has source probability $4^{-3}$ and proposal probability $1/d$, so its importance weight gains the factor $(1/4^3)/(1/d)=d/4^3$ under both \algSMC{} and \algSMCrb{} (\cref{eq:sis-weight-update,eq:weight-update-tlm}).
After $T$ amino acids every particle carries
\begin{equation}
  \weight^{(m)} = \frac{1}{M}\frac{d^{T}}{4^{3T}} = \frac{\pTarget(\tgtStr)}{M},
\end{equation}
so $\estZ = \sum_{m=1}^{M}\weight^{(m)} = \pTarget(\tgtStr)$ exactly.
The factor $d^{T}$ emerges in the importance weights, letting $M$ particles represent $d^{T}$ covering source strings.
Thus \algSMC{} and \algSMCrb{} both have zero variance. Fixed-$M$ \algSWOR{} is also exact. At every branch its child weights are equal, and its without-replacement reweighting preserves their total on every draw.

This is an extreme case, a uniform source and a constant alias count $d$.
In practice, non-uniform DNA models introduce variance across particles.
The example does not distinguish among the stochastic estimators, whose variance differences are measured in the non-uniform experiments.

\section{Details of the MECO Rerun}
\label{app:meco-anatomy}
We compare contextual surprisals computed by beam summing with $\pruneTau$ and \sworHalt{} on the same $2{,}264$ PTB units in the $12$ MECO trials (individually read paragraphs).
We use the generalized additive mixed models (GAMMs) of \citet{kiegeland2025proper}, which model reading times with a log-normal distribution.
The baseline includes unit length and unigram surprisal for the current unit and the two preceding units.
The target model adds contextual surprisal at the same three positions.
Both models include a participant random intercept and by-participant random slopes for every predictor.
We use $12$-fold leave-one-trial-out cross-validation and report per-observation held-out $\dll$.
Confidence intervals come from a $1{,}000$-resample trial-level bootstrap, and significance is assessed with one-sided paired sign-flip tests.
We change only the pruning function used to compute contextual surprisal and rerun the fits on the same observations ($n{=}33{,}472$).
The $\pruneTau$ values are the telescoping estimates reported by \citet{kiegeland2025proper}, computed with a threshold of $5\times10^{-3}$.
For \sworHalt{}, we use $\rho{=}0.1$ and nominal $M{=}8000$, and average each prefix-probability estimate over $10$ runs before taking logarithms.
The reported GAMM fits use the resulting surprisals without a jackknife correction.
The analysis below shows that applying the correction changes no conclusion.\looseness=-1

\paragraph{Jackknife check.}
Averaging the prefix-probability estimates preserves unbiasedness, but taking the logarithm of an average introduces bias.
To estimate it, we recompute each unit's surprisal ten times, omitting one run each time.
If $s$ is one unit's surprisal from all ten runs and $s_{\cancel{r}}$ is its surprisal with run $r$ omitted, the jackknife estimates the bias as $9\bigl(\tfrac{1}{10}\sum_{r=1}^{10}s_{\cancel{r}}-s\bigr)$, which we subtract from $s$.

The correction raises the corpus log probability by $1.55$ nats, equivalently lowering its surprisal by the same amount. It changes the $106$-nat uncorrected difference reported in \cref{sec:psycholing} to approximately $107$ nats. Refitting with the corrected surprisals moves every held-out $\dll$ by at most $0.015\times10^{-3}$ nats and changes neither the results in \cref{tab:gamm_smc_within} nor the matched comparison below.

\paragraph{Matched fit comparison.}
Under the experimental setup of \citet{kiegeland2025proper}, one-sided paired sign-flip tests show no evidence of a predictive difference between the matched PTB-word fits ($\significance\ge 0.068$ in both directions for every measure).

\section{Additional Byte-to-Word Results}
\label{app:perpos-grid}

\Cref{fig:perpos-grid,fig:perpos-grid-b} extend the exact bigram validation of \cref{fig:validation} to every byte position in all ten paragraphs. \Cref{tab:gpt2-head-to-head} reports the additional per-paragraph GPT-2 estimates and runtimes.
\begin{figure}[htbp]
  \centering
  \includegraphics[width=0.9\linewidth]{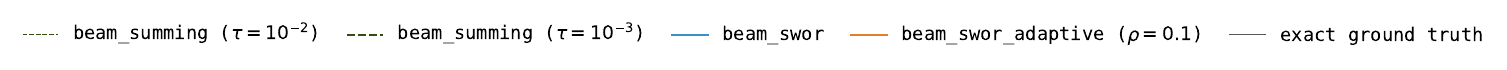}\\[0.6ex]
  \includegraphics[width=0.48\linewidth]{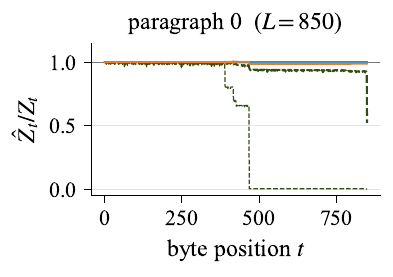}\hfill
  \includegraphics[width=0.48\linewidth]{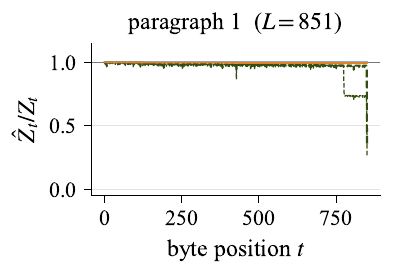}\\[0.4ex]
  \includegraphics[width=0.48\linewidth]{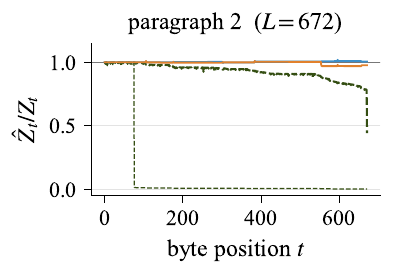}\hfill
  \includegraphics[width=0.48\linewidth]{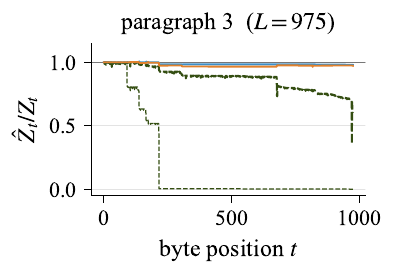}\\[0.4ex]
  \includegraphics[width=0.48\linewidth]{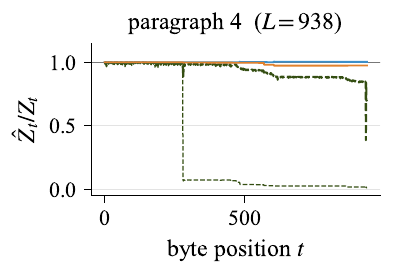}\hfill
  \includegraphics[width=0.48\linewidth]{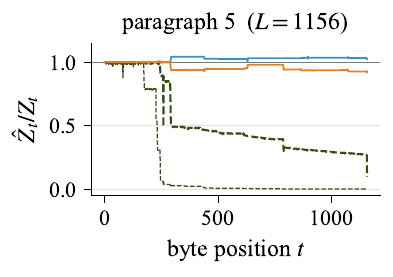}
  \caption{Per-position ratio $\estZ_t / \evZ_t$ for all ten $\pbigram{\circ}\STPTB$ paragraphs at $M{=}200$, averaged over $256$ runs. Paragraphs 0--5 are shown here and 6--9 in \cref{fig:perpos-grid-b}. The \algSWOR{} and \sworB{} averages remain close to one with sampling variation. Threshold-pruned beam summing at $\pruneThreshold{=}10^{-2}$ and $\pruneThreshold{=}10^{-3}$ falls below one where pruning discards positive covering mass.}
  \label{fig:perpos-grid}
\end{figure}

\begin{figure}[htbp]
  \centering
  \includegraphics[width=0.48\linewidth]{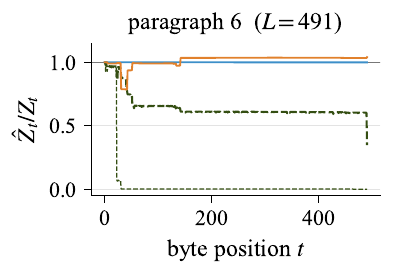}\hfill
  \includegraphics[width=0.48\linewidth]{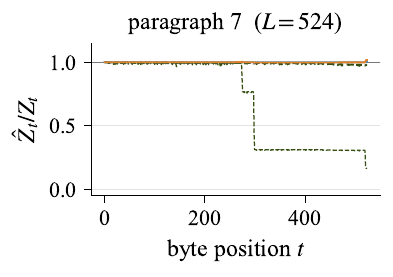}\\[0.4ex]
  \includegraphics[width=0.48\linewidth]{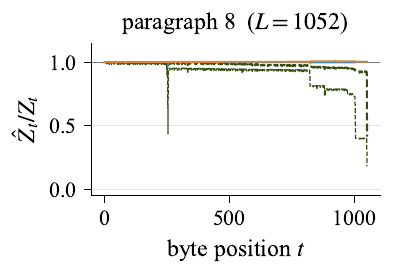}\hfill
  \includegraphics[width=0.48\linewidth]{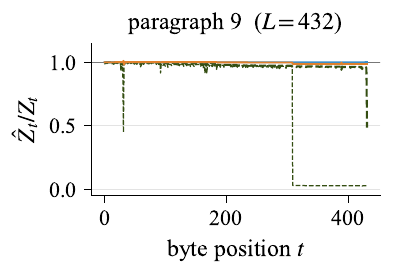}
  \caption{Per-position validation, paragraphs 6--9 (continued from \cref{fig:perpos-grid}), with the same methods and axes.}
  \label{fig:perpos-grid-b}
\end{figure}

\begin{table}[b]
  \centering
  \small
  \caption{Per-paragraph log-prefix probabilities in nats and wall time in seconds under $\gpt{\circ}\STPTB$. \algBeam{} uses $\pruneThreshold{=}10^{-3}$. \sworB{} uses $\rho{=}0.1$, nominal $M{=}800$, and thirty-two runs. Parentheses give the between-run standard deviation, and wall times are per run. The $\Delta$ column is \algBeam{} minus \sworB{}. The total row sums the per-paragraph log estimates. The corpus panel of \cref{fig:gpt2-ptb} instead averages the per-run corpus estimates before taking the logarithm, hence its $-5376.7$.}
  \label{tab:gpt2-head-to-head}
  \begin{tabular}{@{}rrrrrr@{}}
    \toprule
    para  & \algBeam{} & \sworB{}                         & $\Delta$ & $t_{\mathrm{beam}}$ & $t_{\mathrm{adaptive}}$ \\
    \midrule
    0     & $-556.96$   & $-554.96$ {\scriptsize$(0.25)$} & $-2.00$  & $52$                 & $57$                  \\
    1     & $-564.10$   & $-556.41$ {\scriptsize$(1.38)$} & $-7.70$  & $106$                & $210$                 \\
    2     & $-443.86$   & $-443.21$ {\scriptsize$(0.12)$} & $-0.65$  & $54$                 & $102$                 \\
    3     & $-747.96$   & $-743.12$ {\scriptsize$(0.19)$} & $-4.84$  & $170$                & $216$                 \\
    4     & $-635.85$   & $-631.16$ {\scriptsize$(0.20)$} & $-4.69$  & $191$                & $269$                 \\
    5     & $-772.31$   & $-768.42$ {\scriptsize$(0.27)$} & $-3.88$  & $353$                & $647$                 \\
    6     & $-359.56$   & $-355.23$ {\scriptsize$(0.12)$} & $-4.34$  & $190$                & $244$                 \\
    7     & $-330.81$   & $-330.57$ {\scriptsize$(0.13)$} & $-0.24$  & $21$                 & $21$                  \\
    8     & $-687.00$   & $-684.43$ {\scriptsize$(0.13)$} & $-2.57$  & $53$                 & $73$                  \\
    9     & $-311.62$   & $-309.09$ {\scriptsize$(0.27)$} & $-2.54$  & $33$                 & $39$                  \\
    \midrule
    total & $-5410.04$  & $-5376.60$                      & $-33.45$ & $1224$               & $1879$                \\
    \bottomrule
  \end{tabular}
\end{table}

\end{document}